\documentclass{article}

\usepackage{siunitx}
\usepackage{microtype}
\usepackage{graphicx}
\usepackage{subfigure}
\usepackage{tcolorbox}
\usepackage{booktabs} 
\usepackage[normalem]{ulem}
\usepackage{hyperref}

\usepackage{url}

\usepackage{hyperref}
\usepackage{tikz}

\usepackage{booktabs} 
\usepackage{adjustbox}

\definecolor{bluegray}{rgb}{0.4, 0.6, 0.8}
\definecolor{electriclime}{rgb}{0.8, 1.0, 0.0}
\definecolor{malachite}{rgb}{0.04, 0.85, 0.32}
\definecolor{darkred}{rgb}{0.55, 0.0, 0.0}
\definecolor{darkblue}{rgb}{0.0, 0.0, 0.55}
\definecolor{darkgreen}{rgb}{0.0, 0.2, 0.13}
\definecolor{darkorchid}{rgb}{0.6, 0.2, 0.8}

\newcommand{\tikzcircle}[2][red,fill=red]{\tikz[baseline=-0.7ex]\draw[#1,radius=#2] (0,0) circle ;}%
\definecolor{gold}{RGB}{221, 196, 65}
\definecolor{silver}{RGB}{215, 215, 215}
\definecolor{bronze}{RGB}{126, 66, 5}

\usepackage[accepted]{icml2023}

\usepackage{amsmath}
\usepackage{amssymb}
\usepackage{mathtools}
\usepackage{amsthm}

\usepackage[noabbrev,capitalise,nameinlink]{cleveref}
\hypersetup{colorlinks={true},linkcolor={darkred},citecolor=darkblue}

\newcommand{\cheeg}{\mathsf{h}_{\mathsf{Cheeg}
}}
\newcommand{\res}{\mathsf{Res}}
\newcommand{\V}{\mathsf{V}}
\newcommand{\E}{\mathsf{E}}

\newcommand{\R}{\mathbb{R}}
\newcommand{\Hi}{\mathbf{H}}

\newcommand{\up}{\sigma}
\newcommand{\rs}{\mathsf{r}}
\newcommand{\mpas}{\mathsf{a}}
\newcommand{\agg}{\mathsf{agg}}
\newcommand{\com}{\mathsf{com}}
\newcommand{\gph}{\mathsf{G}} 

\newcommand{\MPNN}{\mathsf{MPNN}}
\newcommand{\rew}{\mathcal{R}}
\newcommand{\eigen}{\boldsymbol{\psi}}

\newcommand{\W}{\mathbf{W}}
\newcommand{\DELta}{\boldsymbol{\Delta}}

\newcommand{\OMEga}{\boldsymbol{\Omega}}
\newcommand{\Anorm}{\boldsymbol{\mathsf{A}}}
\newcommand{\tel}{\mathsf{MPNN}_{\mathsf{tel}}}
\newcommand{\poly}{\mathsf{p}}
\newcommand{\ourname}{\text{telescopic}-\mathsf{MPNN}}
\newcommand{\oper}{\boldsymbol{\mathsf{S}}_{\rs,\mpas}}
\newcommand{\obst}{\mathsf{O}}

\theoremstyle{plain}
\newtheorem{theorem}{Theorem}[section]
\newtheorem{proposition}[theorem]{Proposition}

\newtheorem{corollary}[theorem]{Corollary}
\theoremstyle{definition}
\newtheorem{definition}[theorem]{Definition}
\newtheorem{assumption}[theorem]{Assumption}
\theoremstyle{remark}

\usepackage[textsize=tiny]{todonotes}

\icmltitlerunning{On over-squashing in $\MPNN$s: The Impact of Width, Depth, and Topology}

\begin{document}

\twocolumn[
\icmltitle{On Over-Squashing in Message Passing Neural Networks: \\ The Impact of Width, Depth, and Topology}




\begin{icmlauthorlist}
\icmlauthor{Francesco Di Giovanni}{cam}
\icmlauthor{Lorenzo Giusti}{sap}
\icmlauthor{Federico Barbero}{oxf}
\icmlauthor{Giulia Luise}{msr}
\icmlauthor{Pietro Li\`{o}}{cam}
\icmlauthor{Michael Bronstein}{oxf}
\end{icmlauthorlist}

\icmlaffiliation{cam}{University of Cambridge}
\icmlaffiliation{sap}{Sapienza University}
\icmlaffiliation{oxf}{University of Oxford}
\icmlaffiliation{msr}{Microsoft Research}

\icmlcorrespondingauthor{Francesco Di Giovanni}{fd405@cam.ac.uk}

\icmlkeywords{Machine Learning, ICML}

\vskip 0.3in
]



\printAffiliationsAndNotice{} 

\begin{abstract}
Message Passing Neural Networks (MPNNs) are instances of Graph Neural Networks that leverage the graph to send messages over the edges. This inductive bias leads to a phenomenon known as over-squashing, where a node feature is insensitive to information contained at distant nodes. Despite recent methods introduced to mitigate this issue, an understanding of the causes for over-squashing and of possible solutions are lacking. In this theoretical work, we prove that: (i) Neural network width can mitigate over-squashing, but at the cost of making the whole network more sensitive; (ii) Conversely, depth cannot help mitigate over-squashing: increasing the number of layers leads to over-squashing being dominated by vanishing gradients; (iii) The graph topology plays the greatest role, since over-squashing occurs between nodes at high commute time. Our analysis provides a unified framework to study different recent methods introduced to cope with over-squashing and serves as a justification for a class of methods that fall under graph rewiring.
\end{abstract}

\section{Introduction}\label{sec:introduction}
Learning on graphs with Graph Neural Networks (GNNs) \citep{sperduti1994encoding, goller1996learning,gori2005new, scarselli2008graph, bruna2013spectral,defferrard2016convolutional} has become an increasingly flourishing area of machine learning. 
Typically, GNNs operate in the \emph{message-passing paradigm} by exchanging information between nearby nodes \citep{gilmer2017neural}, giving rise to the class of 
Message-Passing Neural Networks ($\MPNN$s). While message-passing has demonstrated to be a useful inductive bias, it has also been shown that the paradigm has some fundamental flaws, from expressivity \citep{xu2018how, morris2019weisfeiler}, to over-smoothing \citep{nt2019revisiting,cai2020note, bodnar2022neural, rusch2022graph,di2022graph, zhao2022analysis} and over-squashing. The first two limitations have been thoroughly investigated, however {\em less is known about over-squashing}. 

\citet{alon2020bottleneck} described over-squashing as an issue emerging when $\MPNN$s propagate messages across distant nodes, with the exponential expansion of the receptive field of a node leading to many messages being `squashed' into fixed-size vectors. \citet{topping2021understanding} formally justified this phenomenon via a sensitivity analysis on the Jacobian of node features and, partly, linked it to the existence of edges with high-negative curvature. However, some important {\bf questions are left open} from the analysis in \citet{topping2021understanding}: (i) What is the impact of {\em width} in mitigating over-squashing? (ii) Can over-squashing be avoided by sufficiently {\em deep} models? (iii) How does over-squashing relate to the graph-spectrum and the underlying {\em topology} beyond curvature bounds that only apply to 2-hop propagation? The last point is particularly relevant due to recent works trying to combat over-squashing via methods that depend on the graph spectrum \citep{arnaiz2022diffwire,deac2022expander,karhadkar2022fosr}. However, it is yet to be clarified if and why these works alleviate over-squashing.

In this work, we aim to address all the questions that are left open in \citet{topping2021understanding} to provide a better theoretical understanding on the causes of over-squashing as well as on what can and cannot fix it. 


\paragraph{Contributions and outline.} An $\MPNN$ is generally constituted by two main parts: a choice of architecture, and an underlying graph over which it operates. In this work, we investigate how these factors participate in the over-squashing phenomenon. We focus on the width and depth of the $\MPNN$, as well as on the graph-topology. 

\begin{itemize}
    \item In \Cref{sec:width}, we prove that the \emph{width} can mitigate over-squashing (\Cref{cor:bound_MLP_MPNN}), albeit at the potential cost of 
    generalization. We also verify this with experiments.  
    \item In \Cref{sec:depth}, we show that depth may not be able to alleviate over-squashing. We identify two regimes. In the first one, the number of layers is comparable to the graph diameter, and we prove that over-squashing is likely to occur among distant nodes (\Cref{cor:over-squasing_distance}). In fact, the distance at which over-squashing happens is strongly dependent on 
    the graph topology -- as we validate experimentally. In the second regime, we consider an arbitrary (large) number of layers. We prove that at this stage the $\MPNN$ is, generally, 
    dominated by vanishing gradients (\Cref{thm:vanishing}). This result is of independent interest, since it characterizes analytically conditions of vanishing gradients of the loss for a large class of $\MPNN$s that also include residual connections.
    \item In \Cref{sec:topology} we show that the \emph{topology} of the graph has the greatest impact on over-squashing. In fact, we show that over-squashing happens among nodes with high commute time (\Cref{thm:effective_resistance}) and we validate this empirically. This provides a unified framework to explain why all spatial and spectral {\em rewiring} approaches (discussed in \Cref{sec:related_work}) do mitigate over-squashing.
\end{itemize}

\section{Background and related work}\label{sec:preliminaries}

\subsection{The message-passing paradigm}
Let $\gph$ be a graph with nodes $\V$ and edges $\E$. The connectivity is encoded in the adjacency matrix $\mathbf{A}\in\R^{n\times n}$, with $n$ the number of nodes. We assume that $\gph$ is undirected, connected, and that there are features $\{\mathbf{h}_{v}^{(0)}\}_{v\in \V}\subset \R^{p}$. 
Graph Neural Networks (GNNs) are functions of the form $\mathsf{GNN}_{\theta}:(\gph,\{\mathbf{h}_v^{(0)}\}) \mapsto y_\gph$, with parameters $\theta$ estimated via training and whose output $y_\gph$ is either a node-level or graph-level prediction. The most studied class of GNNs, known as the Message Passing Neural Network ($\MPNN$) \citep{gilmer2017neural}, compute node representations 
by stacking $m$ layers of the form: 
\begin{align*}
    \mathbf{h}_{v}^{(t)} = \mathsf{com}^{(t)}(\mathbf{h}_{v}^{(t-1)},\mathsf{agg}^{(t)}(\{\mathbf{h}_{u}^{(t-1)}: (v,u)\in \E\})),
\end{align*}
\noindent for $t = 1,\hdots, m$, where  $\mathsf{agg}^{(t)}$ is some {\em aggregation} function invariant to node permutation, while $\com^{(t)}$ {\em combines} the node's current state with messages from its neighbours. In this work, we usually assume $\agg$ to be of the form
\begin{equation}\label{eq:message-passing-operator}
    \agg^{(t)}(\{\mathbf{h}_{u}^{(t-1)}: (v,u)\in \E\}) = \sum_{u}\Anorm_{vu}\mathbf{h}^{(t-1)}_{u},
\end{equation}
\noindent where $\Anorm\in\R^{n\times n}$ is a {\bf Graph Shift Operator (GSO)}, meaning that $\Anorm_{vu} \neq 0$ if and only if $(v,u)\in \mathsf{E}$. Typically, $\Anorm$ is a (normalized) adjacency matrix that we also refer to as message-passing matrix.
\noindent 
While instances of $\MPNN$ differ based on the choices of $\Anorm$ and $\mathsf{com}$, they all aggregate messages over the neighbours, such that in a layer, only nodes connected via an edge exchange messages.
This presents two advantages: $\MPNN$s (i) can capture graph-induced `short-range' dependencies well, 
and (ii) are efficient, since they can leverage the sparsity of the 
graph. 
Nonetheless, $\MPNN$s have been shown to suffer from a few drawbacks, including {\em limited expressive power} and {\em over-squashing}. The problem of expressive power 
stems from the equivalence of $\MPNN$s to the Weisfeiler-Leman graph isomorphism test 
\citep{xu2018how, morris2019weisfeiler}. This framework has been studied extensively 
\citep{jegelka2022theory}.  
On the other hand, the phenomenon of over-squashing, which is the main focus of this work, is more elusive and less understood. We review what is currently known about it in the next subsection. 


\subsection{The problem: introducing over-squashing} Since in an $\MPNN$ the information is aggregated over the neighbours, for a node $v$ to be affected by features at distance $r$, an $\MPNN$ needs at least $r$ layers \citep{barcelo2019logical}. It has been observed though 
that due to the expansion of the receptive field of a node, $\MPNN$s may end up sending a number of messages growing exponentially with the distance $r$, 
leading to a potential loss of information known as {\em over-squashing} \citep{alon2020bottleneck}. 
\citet{topping2021understanding} showed that for an $\MPNN$ with message-passing matrix $\Anorm$ as in Eq.~\eqref{eq:message-passing-operator} and {\em scalar} features, given nodes $v,u$ at distance $r$, we have $\lvert \partial h_{v}^{(r)} / \partial h_{u}^{(0)}\rvert \leq c\cdot (\Anorm^{r})_{vu},$ 
\noindent with $c$ a constant depending on the Lipschitz regularity of the model. If 
$(\Anorm^{r})_{vu}$ decays exponentially with $r$, 
then the feature of $v$ is insensitive to the information contained at $u$. Moreover, \citet{topping2021understanding} showed that over-squashing is related to the existence of edges with {\em high negative curvature}. Such characterization though only applies to propagation of information up to 2 hops. 
\begin{figure}[t]
    \centering
    \includegraphics[width=0.48\textwidth]{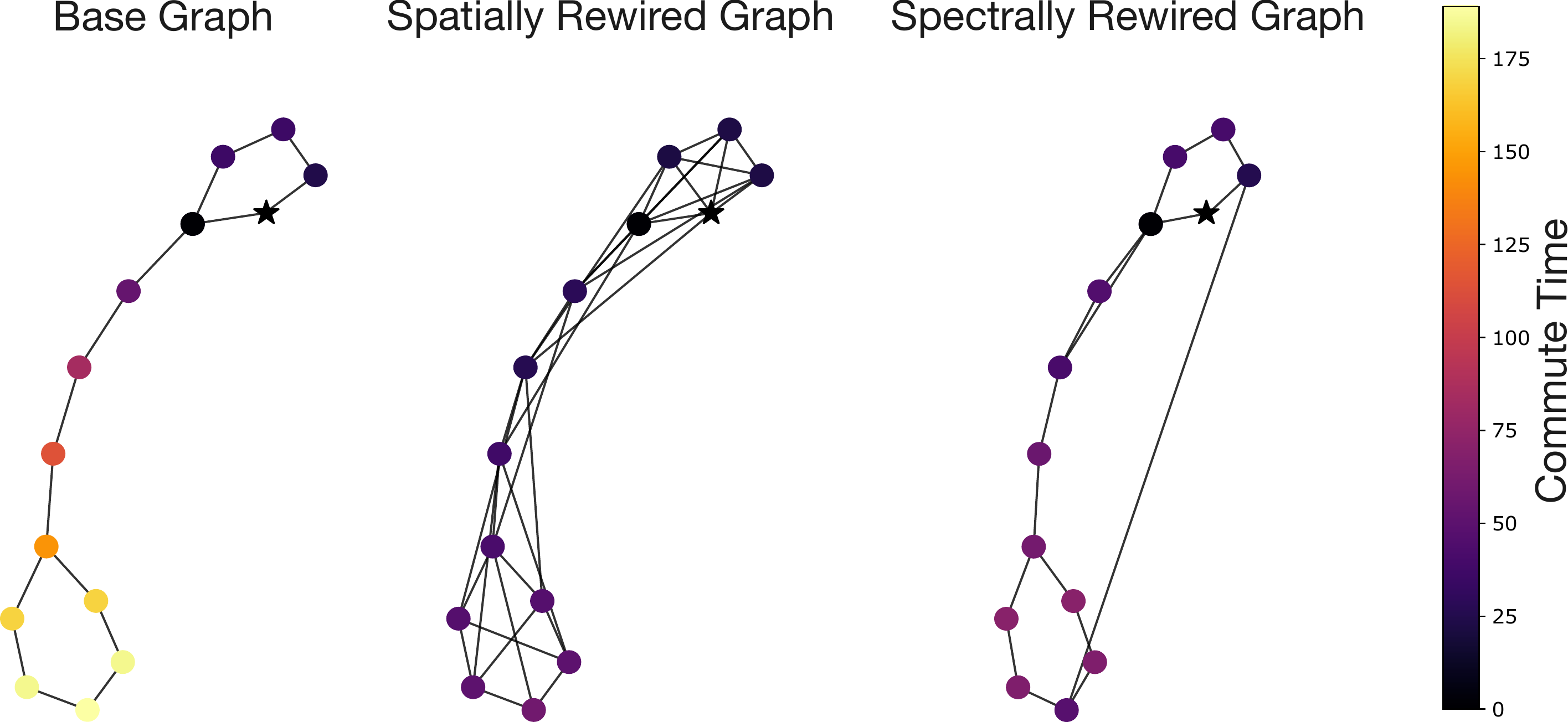}
    \vspace{-1.9mm}
    \caption{Effect of different rewirings $\mathcal{R}$ on the graph connectivity. 
    The colouring 
    denotes Commute Time -- defined in \Cref{sec:topology} -- w.r.t. to the star node. 
    From left to right, the graphs shown are: the base, spatially rewired and spectrally rewired. The added edges significantly reduce the Commute Time and hence mitigate over-squashing in light of \Cref{thm:effective_resistance}. }
    \vspace{-3mm}\label{fig:effective-resistance}
\end{figure}
 
\subsection{Related work}\label{sec:related_work}
Multiple solutions to mitigate over-squashing have already been proposed. We classify them below; in \Cref{sec:topology}, we provide a  unified framework that encompasses all such solutions. We first introduce the following notion:

\begin{definition}\label{def:rewiring}
Consider an $\MPNN$, a graph $\gph$ with adjacency $\mathbf{A}$, and a map $\mathcal{R}:\R^{n\times n}\rightarrow \R^{n\times n}$. We say that $\gph$ has been {\bf rewired} by $\mathcal{R}$, if the messages are exchanged on $\mathcal{R}(\gph)$ instead of $\gph$, with $\mathcal{R}(\gph)$ the graph with adjacency $\mathcal{R}(\mathbf{A})$. 
\end{definition}
\noindent Recent approaches to combat over-squashing share a common idea: replace the graph $\gph$ with a rewired graph $\mathcal{R}(\gph)$ enjoying better connectivity \Cref{fig:effective-resistance}. We then distinguish these works based on the choice of the rewiring $\mathcal{R}$. 

\paragraph{Spatial methods.} Since $\MPNN$s fail to propagate information to distant nodes, a solution consists in replacing $\gph$ with $\rew(\gph)$ such that $\mathrm{diam}(\rew(\gph)) \ll \mathrm{diam}(\gph)$. 
Typically, this is achieved by either explicitly adding edges (possibly attributed) 
between distant nodes 
\citep{bruel2022rewiring,abboud2022shortest,gutteridge2023drew} or by allowing distant nodes to communicate through higher-order structures (e.g., cellular or simplicial complexes, \citep{bodnar2021weisfeilercell,bodnar2021weisfeiler}, which requires additional domain knowledge 
and incurs a computational overhead). 
{\em Graph-Transformers} 
can be seen as an extreme example of rewiring, 
where $\rew(\gph)$ is a {\em complete graph} with edges weighted via attention \citep{kreuzer2021rethinking, mialon2021graphit, ying2021transformers,rampavsek2022recipe}. While these methods do alleviate over-squashing, since they {\em bring all pair of nodes closer}, they come at the expense of making the graph $\rew(\gph)$ much denser. In turn, this has an impact on computational complexity and introduces the risk of mixing local and non-local interactions. 

We include in this group \citet{topping2021understanding} and \citet{banerjee2022oversquashing}, 
where the rewiring is {\em surgical} -- but requires specific pre-processing -- in the sense that 
$\gph$ is replaced by $\rew(\gph)$ where edges have only been added to `mitigate' bottlenecks as identified, for example, by negative curvature \citep{ollivier2007ricci, di2022heterogeneous}. 

We finally mention that spatial rewiring, intended as accessing information beyond the 1-hop when updating node features, is common to many existing frameworks 
\cite{abu2019mixhop,klicpera2019diffusion, chen2020supervised, ma2020path, wang2020multi, nikolentzos2020k}. However, 
this is usually done via powers of the adjacency matrix, which is the main culprit for over-squashing \citep{topping2021understanding}. 
Accordingly, although the diffusion operators $\Anorm^{k}$ allow to aggregate information over non-local hops, they are not suited to mitigate over-squashing. 

\paragraph{Spectral methods.} 
The connectedness of a graph $\gph$ can be measured via a quantity known as the {\em Cheeger constant}, defined as follows \citep{chung1997spectral}:

\begin{definition}\label{def:cheeger}
For a graph $\gph$, the Cheeger constant is
\begin{equation*}
    \cheeg = \min_{\mathsf{U}\subset \V}\frac{\lvert \{(u,v)\in\E: u\in \mathsf{U}, v\in \V\setminus \mathsf{U}\}\rvert}{\min(\mathrm{vol}(\mathsf{U}),\mathrm{vol}(\V\setminus \mathsf{U}))},
\end{equation*}
\noindent where $\mathrm{vol}(\mathsf{U}) = \sum_{u\in\mathsf{U}}d_u$, with $d_u$ the degree of node $u$.
\end{definition}
\noindent The Cheeger constant $\cheeg$ represents the energy required to disconnect $\gph$ into two communities. A small $\cheeg$ means that $\gph$ generally has two communities separated by only few edges -- over-squashing is then expected to occur here {\em if} information needs to travel from one community to the other. While $\cheeg$ is generally intractable to compute, thanks to the Cheeger inequality we know that $\cheeg \sim \lambda_1$, where $\lambda_1$ is 
the positive, smallest eigenvalue of the graph Laplacian. Accordingly, a few new approaches have suggested to choose a rewiring that depends on the spectrum of $\gph$ and yields a new graph satisfying $\cheeg(\rew(\gph)) > \cheeg(\gph)$.
This strategy includes \citet{arnaiz2022diffwire,deac2022expander, karhadkar2022fosr}. 
It is claimed that sending messages over such a graph $\rew(\gph)$ 
alleviates over-squashing, however this has not been shown analytically yet. 

\paragraph{The goal of this work.} The analysis of \citet{topping2021understanding}, which represents our current theoretical understanding of the over-squashing problem, 
leaves some important open questions which we address in this work: 
(i) We study the role of the {\bf width} in mitigating over-squashing; (ii) We analyse what happens when the {\bf depth} exceeds the distance among two nodes of interest; (iii) We prove how over-squashing is related to the graph structure (beyond local curvature-bounds) and its {\bf spectrum}. As a consequence of (iii), we  {\em provide a unified framework to explain how spatial and spectral approaches alleviate over-squashing}. We reference here the concurrent work of \citet{black2023understanding}, who,  similarly to us, drew a strong connection between over-squashing and Effective Resistance (see \Cref{sec:topology}).




\paragraph{Notations and conventions to improve readability.} In the following, to prioritize readability we often leave sharper bounds with more optimal and explicit terms to the Appendix. From now on, $p$ always denotes the width (hidden dimension) while $m$ is the depth (i.e. number of layers). The feature of node $v$ at layer $t$ is written as $\mathbf{h}_{v}^{(t)}$. Finally, we write $[\ell] = \{1,\ldots,\ell\}$ for any integer $\ell$. All proofs can be found in the Appendix.


\section{The impact of width}\label{sec:width}
In this Section we assess whether the width of the underlying $\MPNN$ can mitigate over-squashing and the extent to which this is possible. In order to do that, we extend the sensitivity analysis in \citet{topping2021understanding} to higher-dimensional node features. We consider a class of $\MPNN$s parameterised by neural networks, of the form: 
\begin{equation}\label{eq:MPNN_mlp}
    \mathbf{h}_{v}^{(t+1)} = \up\Big(c_{\rs}\W_{\rs}^{(t)} \mathbf{h}_{v}^{(t)} + c_{\mpas}\W_{\mpas}^{(t)}\sum_{u}\Anorm_{vu}\mathbf{h}^{(t)}_{u}\Big),
\end{equation}
\noindent where $\sigma$ is a pointwise-nonlinearity, $\W_{\rs}^{(t)},\W_{\mpas}^{(t)}\in\R^{p\times p}$ are learnable weight matrices and $\Anorm$ is a graph shift operator. Note that Eq.~\eqref{eq:MPNN_mlp} includes common $\MPNN$s such as $\mathsf{GCN}$ \citep{kipf2016semi}, $\mathsf{SAGE}$ \citep{hamilton2017inductive}, and $\mathsf{GIN}$ \citep{xu2018how}, where $\Anorm$ is one of $\mathbf{D}^{-1/2}\mathbf{A}\mathbf{D}^{-1/2}$, $\mathbf{D}^{-1}\mathbf{A}$ and $\mathbf{A}$, respectively, with $\mathbf{D}$ the diagonal degree matrix. In \Cref{app:sec_width}, we extend the analysis to a more general class of $\MPNN$s (see \Cref{thm:bound_general}), which includes stacking multiple nonlinearities. We also emphasize that the positive scalars $c_{\rs},c_{\mpas}$ represent the weighted contribution of the residual term and of the aggregation term, respectively. To ease the notations, we introduce a family of message-passing matrices that depend on $c_{\rs},c_{\mpas}$. 

\begin{definition}\label{def:oper}
For a graph shift operator $\Anorm$ and constants $c_{\rs},c_{\mpas} > 0$, we define $\oper := c_{\rs}\mathbf{I} + c_{\mpas}\Anorm\in\R^{n\times n}$ to be the message-passing matrix adopted by the $\MPNN$.
\end{definition}

As in \citet{xu2018representation} and \citet{ topping2021understanding}, we study the propagation of information in the $\MPNN$ via the Jacobian of node features after $m$ layers.  

\begin{theorem}[\textbf{Sensitivity bounds}]\label{cor:bound_MLP_MPNN}
Consider an $\MPNN$ as in Eq.~\eqref{eq:MPNN_mlp} for $m$ layers, with $c_{\up}$ the Lipschitz constant of the nonlinearity $\sigma$ and $w$ the maximal entry-value over all weight matrices. 
For $v,u\in\V$ and width $p$, we have 
\begin{equation}\label{eq:upper_bound_mlp}
    \left\| \frac{\partial \mathbf{h}_{v}^{(m)}}{\partial \mathbf{h}_{u}^{(0)}}\right\|_{L_1} \leq  (\underbrace{\vphantom{\oper^m}c_{\up}wp}_{\mathrm{model}})^{m}\underbrace{(\oper^{m})_{vu}}_{\mathrm{topology}},
\end{equation}
with $\oper^{m}$ the $m^{th}$-power of $\oper$ introduced in \Cref{def:oper}.
\end{theorem}
\noindent 
Over-squashing occurs if the right hand side of Eq.~\eqref{eq:upper_bound_mlp} is too small -- this will be related to the distance among $v$ and $u$ in 
\Cref{subsec:shallow}. A small derivative of $\mathbf{h}_v^{(m)}$ with respect to $\mathbf{h}_u^{(0)}$ means that after $m$ layers, {\em the feature at $v$ is mostly insensitive to the information initially contained at $u$}, and hence that messages have not been propagated effectively. \Cref{cor:bound_MLP_MPNN} clarifies how the model can impact over-squashing through (i) its Lipschitz regularity $c_{\up}, w$ and (ii) its width $p$. In fact, given a graph $\gph$ such that $(\oper^{m})_{vu}$ decays exponentially with $m$, the $\MPNN$ can compensate by increasing the width $p$ and the magnitude of $w$ and $c_\sigma$.  
This confirms 
analytically the discussion in \citet{alon2020bottleneck}: \textbf{a larger hidden dimension $p$ does mitigate over-squashing}. However, this is not an optimal solution since increasing the contribution of the model (i.e. the term $c_{\up}wp$) may lead to over-fitting and poorer generalization \citep{bartlett2017spectrally}. Taking larger values of $c_{\up},w,p$ affects the model {\em globally} and does not target the sensitivity of specific node pairs induced by the topology via $\oper$. 



\paragraph{Validating the theoretical results.}

 We validate empirically the message from \Cref{cor:bound_MLP_MPNN}: if the task presents long-range dependencies, increasing the hidden dimension mitigates over-squashing and therefore has a positive impact on the performance. We consider the following `graph transfer' task, building upon \citet{bodnar2021weisfeilercell}: given a graph, consider source and target nodes, placed at distance $r$ from each other. We assign a one-hot encoded label to the target and a constant unitary feature vector to all other nodes. The goal is to assign to the source node the feature vector of the target. Partly due to over-squashing, performance is expected to degrade as $r$ increases. 

 To validate that this holds irrespective of the graph structure, we evaluate across three graph topologies, called  $\mathsf{CrossedRing}$, $\mathsf{Ring}$ and $\mathsf{CliquePath}$ - see \Cref{app:graph-transfer} for further details. While the topology is also expected to affect the performance (as confirmed in \Cref{sec:depth}), given a fixed topology, we expect the model to benefit from an increase of hidden dimension.
 
To verify this behaviour, we evaluate $\mathsf{GCN}$ \citep{kipf2016semi} on the three graph transfer tasks increasing the hidden dimension, but keeping the number of layers equal to the distance between source and target, as shown in Figure \ref{fig:graph-transfer-hidden-dim}. The results verify the intuition from the theorem that a higher hidden dimension helps the $\mathsf{GCN}$ model solve the task to larger distances 
across the three graph-topologies.

\begin{figure}[!htbp]
    \centering
    \includegraphics[width=0.48\textwidth]{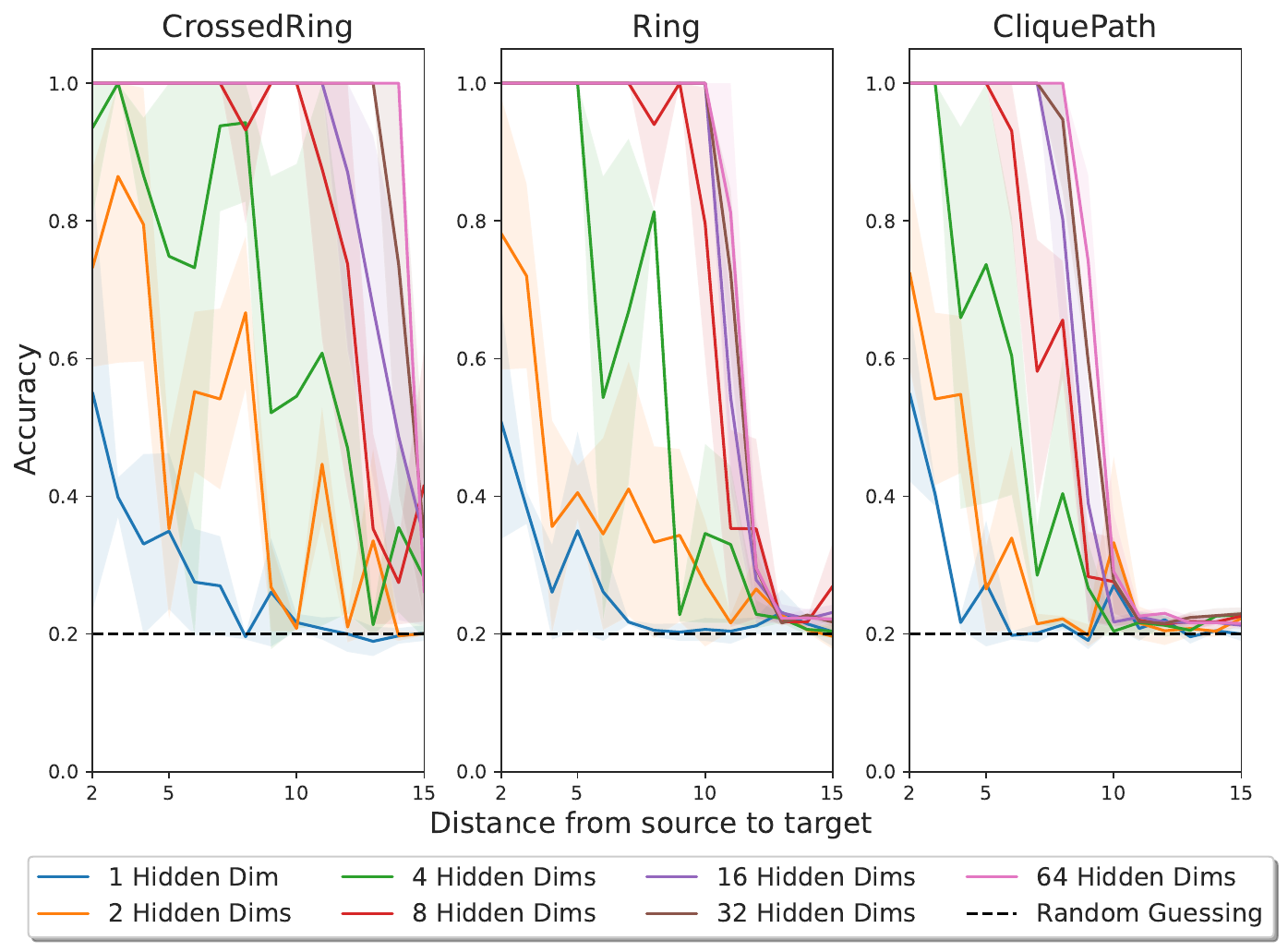}
    \caption{Performance of $\mathsf{GCN}$ on the $\mathsf{CrossedRing}$, $\mathsf{Ring}$, and $\mathsf{CliquePath}$ tasks obtained by varying the hidden dimension. Increasing the hidden dimension helps mitigate the over-squashing effect, in accordance with \Cref{cor:bound_MLP_MPNN}.}
    \label{fig:graph-transfer-hidden-dim}
\end{figure}


\begin{tcolorbox}[boxsep=0mm,left=2.5mm,right=2.5mm]
\textbf{Message of the Section:} {\em The Lipschitz regularity, weights, and width of the underlying $\MPNN$ can help mitigate the effect of over-squashing. However, this is a remedy that comes at the expense of generalization and does not address the real culprit behind over-squashing: the graph-topology}.   
\end{tcolorbox}
\vspace{-5pt}

\section{The impact of depth}\label{sec:depth}
Consider a graph $\gph$ and 
a task with `long-range' dependencies, meaning that there exists (at least) a node $v$ whose embedding has to account for information contained at some node $u$ at distance $r \gg 1$. One natural attempt at resolving over-squashing amounts to increasing the number of layers $m$ to compensate for the distance. We prove that the depth of the $\MPNN$ will, generally, not help with over-squashing. We show that: (i) When the depth $m$ is comparable to the distance, over-squashing is bound to occur among distant nodes -- in fact, the distance at which over-squashing happens, is strongly dependent on the underlying topology; 
(ii) If we take a large number of layers to cover the long-range interactions, we rigorously prove under which {\em exact} conditions, $\MPNN$s incur the vanishing gradients problem.



\subsection{The shallow-diameter regime: over-squashing occurs among distant nodes}\label{subsec:shallow}

Consider the scenario above, with two nodes $v,u$, whose interaction is important for the task, at distance $r$. We first focus on the regime $m \sim r$. We refer to this as the {\em shallow-diameter} regime, since the number of layers $m$ is comparable to the diameter of the graph. 

From now on, we set $\Anorm = \mathbf{D}^{-1/2}\mathbf{A}\mathbf{D}^{-1/2}$, where we recall that $\mathbf{A}$ is the adjacency matrix and $\mathbf{D}$ is the degree matrix. This is not restrictive, but allows us to derive more explicit bounds and, later, bring into the equation the spectrum of the graph. We note that results can be extended easily to $\mathbf{D}^{-1}\mathbf{A}$, given that this matrix is similar to $\Anorm$, and, in expectation, to $\mathbf{A}$ by normalizing the Jacobian as in \citet{xu2018how} and Section A in the Appendix of \citet{topping2021understanding}. 


\begin{theorem}[\textbf{Over-squashing among distant nodes}]\label{cor:over-squasing_distance} Given an $\MPNN$ as in Eq.~\eqref{eq:MPNN_mlp}, with $c_{\mpas} \leq 1$, let $v,u\in\mathsf{V}$ be at distance $r$. Let $c_{\up}$ be the Lipschitz constant of $\sigma$, $w$ the maximal entry-value over all weight matrices, $d_{\mathrm{min}}$ the minimal degree of $\gph$, and $\gamma_{\ell}(v,u)$ the number of walks from $v$ to $u$ of maximal length $\ell$. For any $0 \leq k < r$, there exists $C_{k} > 0$ {\bf independent} of $r$ and of the graph, such that 
\begin{equation}\label{eq:cor_distance}
   \left\| \frac{\partial \mathbf{h}_{v}^{(r+k)}}{\partial \mathbf{h}_{u}^{(0)}}\right\|_{L_1} \leq C_{k}\gamma_{r+k}(v,u)\Big(\frac{2c_{\up}wp}{d_{\mathrm{min}}}\Big)^r.
\end{equation}
\end{theorem}
To understand the bound above, fix $k < r$ and assume that nodes $v,u$ are `badly' connected, meaning that the number of walks $\gamma_{r+k}(v,u)$ of length at most $r+k$, is small. If $2\, c_{\up}wp < d_{\mathrm{min}}$, then the bound on the Jacobian in Eq.~\eqref{eq:cor_distance} {\em decays exponentially with the distance} $r$. 
Note that the bound above considers $d_{\mathrm{min}}$ and $\gamma_{r+k}$ as a worst case scenario. If one has a better understanding of the topology of the graph, sharper bounds can be derived by estimating $(\oper^{r})_{vu}$. 
\Cref{cor:over-squasing_distance} implies that, when the depth $m$ is comparable to the diameter of $\gph$,
\textbf{\em over-squashing becomes an issue if the task depends on the interaction of nodes $v,u$ at `large' distance $r$.} 
In fact, 
\Cref{cor:over-squasing_distance} 
shows that the distance at which the Jacobian sensitivity falls below a given threshold, depends on both the model, via $c_{\up}, w,p$, and on the graph, through $d_{\mathrm{min}}$ and $\gamma_{r+k}(v,u)$. We finally observe that \Cref{cor:over-squasing_distance} generalizes the analysis in \citet{topping2021understanding} in multiple ways: (i) it holds for any width $p > 1$; (ii) it includes cases where $m > r$; (iii) it provides explicit estimates in terms of number of walks and degree information. 

{\bf Remark.} What if $2c_{\up}wp > d_{\mathrm{min}}$? Taking larger weights and hidden dimension increases the sensitivity of node features. However, this occurs {\em everywhere} in the graph the same. Accordingly, nodes at shorter distances will, on average, still have sensitivity exponentially larger than nodes at large distance. This is validated in our synthetic experiments below, where we do not have constraints on the weights.


\paragraph{Validating the theoretical results.}
From Theorem \ref{cor:over-squasing_distance}, we derive a strong indication of the difficulty of a task by calculating an upper bound on the Jacobian. We consider the same graph transfer tasks introduced above, namely $\mathsf{CrossedRing}$, $\mathsf{Ring}$, and $\mathsf{CliquePath}$. For these special cases, we can derive a refined version of the r.h.s in Eq.~\eqref{eq:cor_distance}: in particular, $k = 0$ (i.e. the depth coincides with the distance among source and target) and the term $\gamma_{r}(v,u)(d_{\mathrm{min}})^{-r}$ can be replaced by the exact quantity $(\oper^{r})_{vu}$. 
Fixing a distance $r$ between source $u$ and target $v$ then, if we consider for example the $\mathsf{GCN}$-case, we have $\oper = \Anorm$ so that the term $(\oper^{r})_{vu}$ can be computed explicitly: 
\begin{alignat*}{2}
&(\oper^{r})_{vu} = (3/2)^{-(r-1)} \qquad &&\text{ for } \mathsf{CrossedRing}\\
&(\oper^{r})_{vu} =2^{-(r-1)} \quad &&\text{ for }  \mathsf{Ring}\\ 
& (\oper^{r})_{vu} = 2^{-(r-2)}/(r\sqrt{r-2}) &&\text{ for } 
 \mathsf{CliquePath}.
\end{alignat*}
Given 
an $\MPNN$, terms like $c_{\up},w,p$ entering \Cref{cor:over-squasing_distance} are independent of the graph-topology and hence can be assumed to behave, roughly, the same across different graphs.  
As a consequence, we can expect over-squashing to be more problematic for $\mathsf{CliquePath}$, followed by $\mathsf{Ring}$, and less prevalent comparatively in $\mathsf{CrossedRing}$.
Figure \ref{fig:graph-transfer-main} shows the behaviour of $\mathsf{GIN}$ \citep{xu2018how}, $\mathsf{SAGE}$ \citep{hamilton2017inductive}, $\mathsf{GCN}$ \citep{kipf2016semi}, and $\mathsf{GAT}$ \citep{velivckovic2017graph} on the aformentioned tasks. We verify the conjectured difficulty. $\mathsf{CliquePath}$ is the consistently hardest task, followed by $\mathsf{Ring}$, and $\mathsf{CrossedRing}$. Furthermore, the decay of the performance to random guessing for the {\em same} architecture across different graph topologies highlights that this drop cannot be simply labelled as `vanishing gradients' since for certain topologies the same model can, in fact, achieve perfect accuracy. This validates that the underlying topology has a strong impact on the distance at which over-squashing is expected to happen. Moreover, we confirm that in the regime where the depth $m$ is comparable to the distance $r$, over-squashing will occur if $r$ is large enough. 


\begin{figure}[!htbp]
    \centering
    \includegraphics[width=0.48\textwidth]{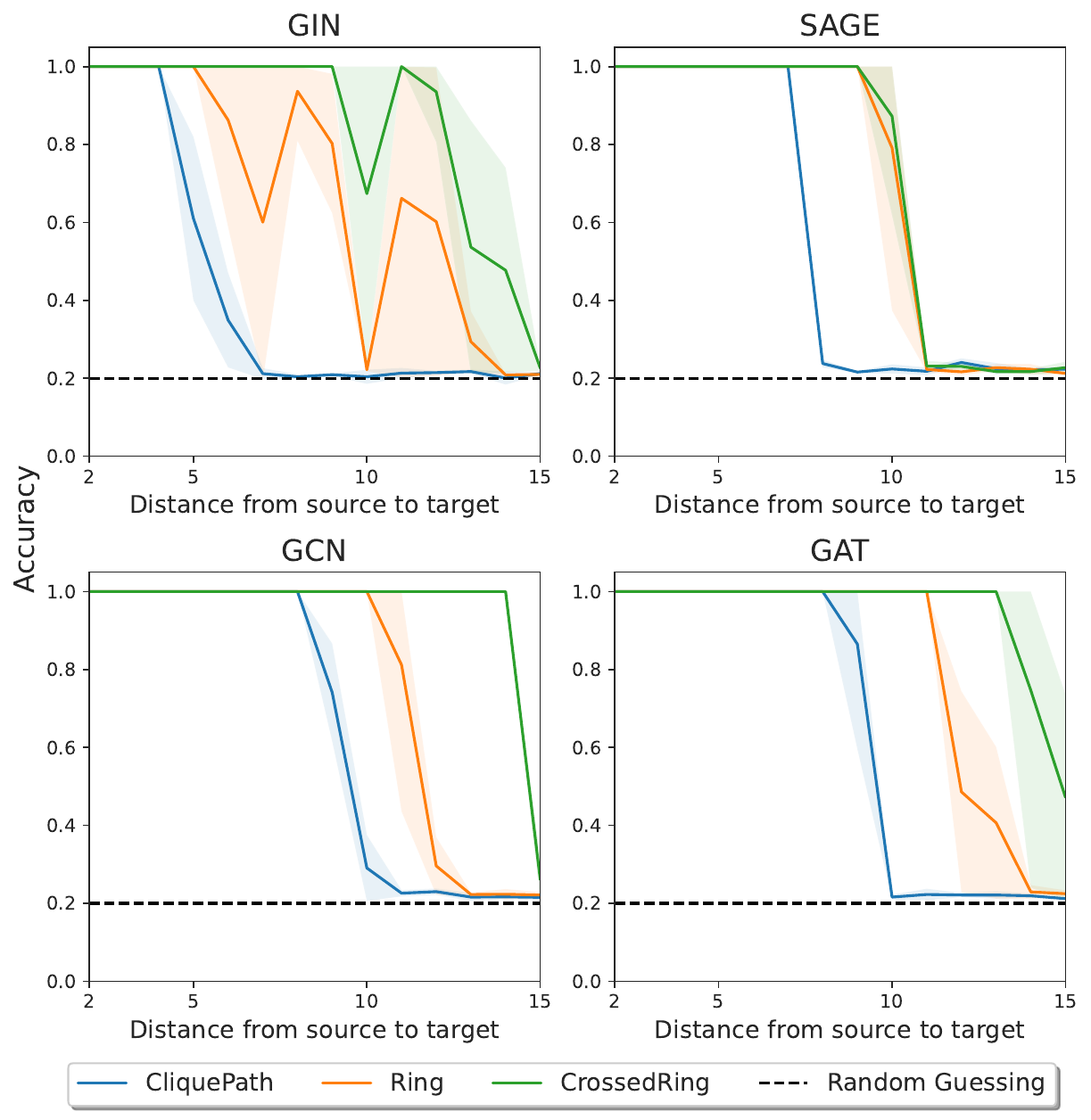}
    \caption{Performance of $\mathsf{GIN}$, $\mathsf{SAGE}$, $\mathsf{GCN}$, and $\mathsf{GAT}$ on the $\mathsf{CliquePath}$, $\mathsf{Ring}$, and $\mathsf{CrossedRing}$ tasks. In the case where depth and distance are comparable, over-squashing highly depends on the topology of the graph as we increase the distance.}
    \label{fig:graph-transfer-main}
\end{figure}

\subsection{The deep regime: vanishing gradients dominate}\label{subsec:deep}
We now focus on the regime where the number of layers $m \gg r$ is 
large. 
We show that in this case, vanishing gradients can occur and make the entire model insensitive. 
Given a weight $\theta^{(k)}$ entering a layer $k$, one can write the gradient of the loss after $m$ layers as \citep{pascanu2013difficulty}
\begin{equation}
    \frac{\partial \mathcal{L}}{\partial \theta^{(k)}} = \sum_{v,u\in V}\Big(\frac{\partial \mathcal{L}}{\partial \mathbf{h}^{(m)}_v}\frac{\partial \mathbf{h}_u^{(k)}}{\partial \vphantom{\mathbf{h}^{(k)}_u}\theta^{(k)}}\Big)\underbrace{\frac{\partial \mathbf{h}_v^{(m)}}{\partial \mathbf{h}_{u}^{(k)}}}_{\mathrm{sensitivity}}
\end{equation}
\noindent We 
provide {\bf exact conditions} for $\MPNN$s to incur the vanishing gradient problem, intended as the gradients of the loss decaying exponentially with the number of layers $m$. 
\begin{theorem}[\textbf{Vanishing gradients}]\label{thm:vanishing} Consider an $\MPNN$ as in Eq.~\eqref{eq:MPNN_mlp} for $m$ layers with a quadratic loss $\mathcal{L}$. Assume that (i) $\sigma$ has Lipschitz constant $c_{\up}$ and $\sigma(0) = 0$, and (ii) weight matrices have spectral norm bounded by $\mu > 0$. 
Given any weight $\theta$ entering a layer $k$, there exists a constant $C > 0$ independent of $m$, such that
\begin{align}
    \left\vert \frac{\partial \mathcal{L}}{\partial \theta}\right\vert &\leq C\left(c_{\up}\mu(c_{\rs} + c_{\mpas})\right)^{m-k}\left(1  +\left(c_{\up}\mu(c_{\rs} + c_{\mpas})\right)^{m}\right).
\end{align}
\noindent In particular, if $c_{\up}\mu(c_{\rs} + c_{\mpas}) < 1$, then the gradients of the loss decay to zero exponentially fast with $m$.
\end{theorem}

\noindent The problem of vanishing gradients for graph convolutional networks have been studied from an empirical perspective \citep{li2019deepgcns,li2021training}. \Cref{thm:vanishing} provides sufficient conditions for the vanishing of gradients to occur in a large class of $\MPNN$s that also include (a form of) residual connections through the contribution of $c_{\rs}$ in Eq.~\eqref{eq:MPNN_mlp}. This extends a behaviour studied for Recurrent Neural Networks \citep{bengio1994learning,hochreiter1997long, pascanu2013difficulty, rusch2021coupled, rusch2021unicornn} to the $\MPNN$ class. We also mention that some discussion on vanishing gradients for $\MPNN$s can be found in \citet{ruiz2020gated} and \citet{rusch2022graph}. A few final comments are in order. (i) The bound in \Cref{thm:vanishing} seems to `hide' the contribution of the graph. This is, in fact, because the spectral norm of the graph operator $\oper$ is $c_{\rs} + c_{\mpas}$ -- we reserve the investigation of more general graph shift operators \citep{dasoulas2021learning} to future work. (ii) \Cref{cor:over-squasing_distance} shows that if the distance $r$ is large enough and we take the number of layers $m$ s.t. $m \sim r$, over-squashing arises among nodes at distance $r$. Taking the number of layers large enough though, may incur the vanishing gradient problem 
\Cref{thm:vanishing}. 
In principle, there might be an intermediate regime where $m$ is larger than $r$, but {\em not} too large, in which the depth could help with over-squashing before it leads to vanishing gradients. Given a graph $\gph$, and bounds on the Lipschitz regularity and width, we conjecture though that there exists $\tilde{r}$, depending on the topology of $\gph$, such that if the task has interactions at distance $r > \tilde{r}$, no number of layers can allow the $\MPNN$ class to solve it. This is left for future work. 



\begin{tcolorbox}[boxsep=0mm,left=2.5mm,right=2.5mm]
\textbf{Message of the Section:} \em Increasing the depth $m$ will, in general, not fix over-squashing. As we increase $m$, 
$\MPNN$s transition from over-squashing (\Cref{cor:over-squasing_distance}) to vanishing gradients (\Cref{thm:vanishing}). 
\end{tcolorbox}
\vspace{-5pt}

\section{The impact of topology }\label{sec:topology}
We finally discuss the impact of the graph topology, and in particular of the graph spectrum, on over-squashing. This allows us to draw a unified framework that shows why existing approaches manage to alleviate over-squashing by either spatial or spectral rewiring (\Cref{sec:related_work}).  

\subsection{On over-squashing and access time} Throughout the section we relate over-squashing to well-known properties of random walks on graphs. To this aim, we first review basic concepts about random walks.

\paragraph{Access and commute time.} A Random Walk (RW) on $\gph$ is a Markov chain which, at each step, moves from a node $v$ to one of its neighbours with probability $1/d_v$. 
Several properties about RWs have been studied. We are particularly interested in the notion of {\em access time} $\mathsf{t}(v,u)$ and of {\em commute time} $\tau(v,u)$ (see \Cref{fig:effective-resistance}). The access time $\mathsf{t}(v,u)$ (also known as {\em hitting time}) is the  expected number of steps
before node $u$ is visited for a RW starting from node $v$. The commute time instead, represents the expected number of steps in a
RW starting at $v$ to reach node $u$ and {\em come back}. A high access (commute) time means that nodes $v,u$ generally struggle to visit each other in a RW -- this can happen if nodes are far-away, but it is in fact more general and strongly dependent on the topology.

Some connections between over-squashing and the topology have already been derived (\Cref{cor:over-squasing_distance}), but up to this point `topology' has entered the picture through `distances' only. 
In this section, we further link over-squashing to other quantities related to the topology of the graph, such as access time, commute time and the Cheeger constant. We ultimately provide a unified framework to understand how existing approaches manage to mitigate over-squashing via graph-rewiring.


\paragraph{Integrating information across different layers.} We consider a family of $\MPNN$s of the form
\begin{equation}\label{eq:mpnn_simplified}
    \mathbf{h}_{v}^{(t)} = \mathsf{ReLU}\Big(\W^{(t)}\Big(c_{\rs}\mathbf{h}_v^{(t-1)} + c_{\mpas}(\Anorm\mathbf{h}^{(t-1)})_v\Big)\Big).
\end{equation}
\noindent Similarly to \citet{kawaguchi2016deep,xu2018representation}, we require the following:


\begin{assumption}\label{assumption_main_body} All paths in the computation graph
of the model are activated with the same probability of
success $\rho$.
\end{assumption}
Take two nodes $v\neq u$ at distance $r\gg 1$ and imagine we are interested in sending information {\em from $u$ to $v$}. Given a layer $k < m$ of the $\MPNN$, by \Cref{cor:over-squasing_distance} we expect that $\mathbf{h}_v^{(m)}$ is much more sensitive to the information contained {\em at the same} node $v$ at an earlier layer $k$, i.e. $\mathbf{h}_{v}^{(k)}$, rather than to the information contained at a distant node $u$, i.e. $\mathbf{h}_{u}^{(k)}$. 
\noindent Accordingly, we introduce the following quantity: 
\begin{align*}
\mathbf{J}_{k}^{(m)}(v,u) := \frac{1}{d_v}\frac{\partial \mathbf{h}_{v}^{(m)}}{\partial \mathbf{h}_{v}^{(k)}} - \frac{1}{\sqrt{d_v d_u}}\frac{\partial \mathbf{h}_{v}^{(m)}}{\partial \mathbf{h}_{u}^{(k)}}. 
\end{align*}
\noindent We note that the normalization by degree stems from our choice $\Anorm = \mathbf{D}^{-1/2}\mathbf{A}\mathbf{D}^{-1/2}$. We provide an intuition for this term. 
Say that node $v$ at layer $m$ of the $\MPNN$ is mostly insensitive to the information sent from $u$ at layer $k$. Then, on average, we expect $\|\partial\mathbf{h}_v^{(m)} / \partial \mathbf{h}_u^{(k)} \| \ll \| \partial\mathbf{h}_v^{(m)} / \partial \mathbf{h}_v^{(k)} \| $. In the opposite case instead, we expect, on average, that $\| \partial\mathbf{h}_v^{(m)} / \partial \mathbf{h}_u^{(k)} \| \sim \| \partial\mathbf{h}_v^{(m)} / \partial \mathbf{h}_v^{(k)} \| $. Therefore $\| \mathbf{J}^{(m)}_k(v,u)\|$ will be {\em larger} when $v$ is (roughly) independent of the information contained at $u$ at layer $k$. We extend the same argument by accounting for messages sent at each 
layer $k \leq m$. 
\begin{definition}\label{def:obst}
The Jacobian obstruction of node $v$ with respect to node $u$ after $m$ layers is $
    \obst^{(m)}(v,u) =  \sum_{k = 0}^{m}\|\mathbf{J}_k^{(m)}(v,u)\|.
$
\end{definition}
\noindent 
\noindent 
As motivated above, a larger $\obst^{(m)}(v,u)$ means that, after $m$ layers, the representation of node $v$ is more likely to be insensitive to information contained at $u$ and conversely, a small $\obst^{(m)}(v,u)$ means that nodes $v$ is, on average, able to receive information from $u$. 
Differently from the Jacobian bounds of the earlier sections, here we consider the contribution coming from all layers $k \leq m$ (note the sum over layers $k$ in Definition \ref{def:obst}). 
\begin{theorem}[\textbf{Over-squashing and access-time}]\label{thm:access} Consider an $\MPNN$ as in Eq.~\eqref{eq:mpnn_simplified} and let Assumption \ref{assumption_main_body} hold. If $\nu$ is the smallest singular value across all weight matrices and $c_{\rs},c_{\mpas}$ are such that $\nu (c_{\rs} + c_{\mpas}) = 1$, then, 
in expectation, we have  
\[
\obst^{(m)}(v,u) \geq \frac{\rho}{\nu c_{\mpas}}\frac{\mathsf{t}(u,v)}{2\lvert \E\rvert}  + o(m), 
\]
\noindent with $o(m)\rightarrow 0$ exponentially fast with $m$. 

\end{theorem} 
\noindent We note that an exact expansion of the term $o(m)$ is reported in the Appendix. 
We also observe that more general bounds are possible if $\nu(c_{\rs} + c_{\mpas}) < 1$ -- however, they will progressively become less informative in the limit $\nu (c_{\rs} + c_{\mpas}) \rightarrow 0$. \Cref{thm:access} shows that the obstruction is a function of the access time $\mathsf{t}(u,v)$; {\bf high access time, on average, translates into high obstruction for node $v$ to receive information from node $u$ inside the $\MPNN$}. This resonates with the intuition that access time is a measure of how easily a `diffusion' process starting at $u$ reaches $v$. We emphasize that the obstruction provided by the access time cannot be fixed by increasing the number of layers and in fact this is independent of the number of layers, further corroborating the analysis in \Cref{sec:depth}. 
Next, we relate over-squashing to commute time, and hence, to effective resistance.

\subsection{On over-squashing and commute time}\label{subsec:effective}
We now restrict our attention to a slightly more special form of over-squashing, where we are interested in nodes $v,u$ exchanging information both ways -- differently from before where we looked at node $v$ receiving information from node $u$. Following the same intuition described previously, we introduce the symmetric quantity: 
\begin{align*}
\tilde{\mathbf{J}}_{k}^{(m)}(v,u) &:= \Big(\frac{1}{d_v}\frac{\partial \mathbf{h}_{v}^{(m)}}{\partial \mathbf{h}_{v}^{(k)}} - \frac{1}{\sqrt{d_v d_u}}\frac{\partial \mathbf{h}_{v}^{(m)}}{\partial \mathbf{h}_{u}^{(k)}}\Big)
\\ &+ \Big(\frac{1}{d_u}\frac{\partial \mathbf{h}_{u}^{(m)}}{\partial \mathbf{h}_{u}^{(k)}}  - \frac{1}{\sqrt{d_v d_u}}\frac{\partial \mathbf{h}_{u}^{(m)}}{\partial \mathbf{h}_{v}^{(k)}}\Big). 
\end{align*}
Once again, we expect that $\|\tilde{\mathbf{J}}^{(m)}_k(v,u)\|$ is larger if nodes $v,u$ are failing to communicate in the $\MPNN$, and conversely to be smaller whenever the communication is sufficiently robust. Similarly, we integrate the information collected at each layer $k\leq m$.
\begin{definition}
The symmetric Jacobian obstruction of nodes $v,u$ after $m$ layers is $
    \tilde{\obst}^{(m)}(v,u) =  \sum_{k = 0}^{m}\|\tilde{\mathbf{J}}_k^{(m)}(v,u)\|.
$
\end{definition}
\noindent 
\noindent 
\noindent The intuition of comparing the sensitivity of a node $v$ with a different node $u$ and to itself, and then swapping the roles of $v$ and $u$, resembles the concept of commute time $\tau(v,u)$. In fact, this is not a coincidence:
\noindent

\begin{theorem}[\textbf{Over-squashing and commute-time}]\label{thm:effective_resistance} Consider an $\MPNN$ as in Eq.~\eqref{eq:mpnn_simplified} with $\mu$ the maximal spectral norm of the weight matrices and $\nu$ the minimal singular value. Let Assumption \ref{assumption_main_body} hold. 
If $\mu (c_{\rs} + c_{\mpas}) \leq 1$, then there exists $\epsilon_\gph$, independent of nodes $v,u$, such that in expectation, we have  
    \begin{align*}
\epsilon_\gph(1 - o(m))\frac{\rho}{\nu c_{\mpas}}\frac{\tau(v,u)}{2\lvert\E\rvert} \leq \tilde{\obst}^{(m)}(v,u) \leq \frac{\rho}{\mu c_{\mpas}}\frac{\tau(v,u)}{2\lvert\E\rvert},
    \end{align*}
with $o(m)\rightarrow 0$ exponentially fast with $m$ increasing.
\end{theorem}
We note that an explicit expansion of the $o(m)$-term is reported in the proof of the Theorem in the Appendix.
\noindent By the previous discussion, a {\bf smaller} $\tilde{\obst}^{(m)}(v,u)$ means 
$v$ is more sensitive to $u$ in the $\MPNN$ (and viceversa when $\tilde{\obst}^{(m)}(v,u)$ is large). Therefore, \Cref{thm:effective_resistance} implies that nodes at small commute time will exchange information better in an $\MPNN$ and conversely for those at high commute time.
This has some {\bf important consequences}:
\begin{itemize}
    \item [(i)] When the task only depends on local interactions, the property of $\MPNN$ of reducing the sensitivity to messages from nodes with high commute time {\em can} be beneficial since it decreases harmful redundancy. 
    \item[(ii)] Over-squashing is an issue when the task depends on the interaction of nodes with high commute time. 
    \item[(iii)] The commute time represents an obstruction to the sensitivity of an $\MPNN$ which is {\em independent of the number of layers}, since the bounds in  \Cref{thm:effective_resistance} are independent of $m$ (up to errors decaying exponentially fast with $m$). 
\end{itemize}

\noindent We note that the very same comments hold in the case of access time as well if, for example, the task depends on node $v$ receiving information from node $u$ but not on $u$ receiving information from $v$.
\subsection{A unified framework}\label{subsec:unified} 

\paragraph{Why spectral-rewiring works.}First, we justify why the spectral approaches discussed in \Cref{sec:related_work} mitigate over-squashing. This comes as a consequence of \citet{lovasz1993random} and \Cref{thm:effective_resistance}:
\begin{corollary}\label{cor:spectral_methods}
 Under the assumptions of \Cref{thm:effective_resistance}, for any $v,u\in\mathsf{V}$, we have 
 \begin{equation*}
 \tilde{\obst}^{(m)}(v,u)\leq \frac{4}{\rho\mu c_{\mpas}}\frac{1}{\cheeg^2}.
 \end{equation*}
\end{corollary}
\noindent \Cref{cor:spectral_methods} essentially tells us that the obstruction among {\em all} pairs of nodes decreases (so better information flow) if the $\MPNN$ operates on a graph $\gph$ with larger Cheeger constant. This rigorously justifies why recent works like \citet{arnaiz2022diffwire, deac2022expander, karhadkar2022fosr} manage to alleviate over-squashing by 
propagating information on a rewired graph $\rew(\gph)$ with larger Cheeger constant $\cheeg$. 
Our result also highlights why bounded-degree expanders are particularly suited 
- as leveraged in \citet{deac2022expander} -- given that their commute time is only $\mathcal{O}(\lvert\E\rvert)$ \citep{chandra1996electrical}, making the bound in \Cref{thm:effective_resistance} scale as $\mathcal{O}(1)$ w.r.t. the size of the graph. In fact, the concurrent work of \citet{black2023understanding} leverages directly the effective resistance of the graph $\mathsf{Res}(v,u) = \tau(v,u)/2\lvert\mathsf{E}\rvert$ to guide a rewiring that improves the graph connectivity and hence mitigates over-squashing.

\paragraph{Why spatial-rewiring works.} 
\citet{chandra1996electrical} proved that the commute time satisfies: $\tau(v,u) = 2\lvert \mathsf{E}\rvert \res(v,u)$, with $\res(v,u)$ the {\bf effective resistance} of nodes $v,u$. 
$\res(v,u)$ measures the voltage difference between nodes $v,u$ if a unit current flows through the graph from $v$ to $u$ and we take each edge to represent a unit resistance \citep{thomassen1990resistances,dorfler2018electrical}, and has also been used in \citet{velingker2022affinity} as a form of structural encoding. Therefore, we emphasize that \Cref{thm:effective_resistance} can be \textbf{\em equivalently rephrased as saying that nodes at high-effective resistance struggle to exchange information in an $\MPNN$} and viceversa for node at low effective resistance. We recall that a result known as Rayleigh's monotonicity principle \citep{thomassen1990resistances}, asserts that the \emph{total} effective resistance $\res_{\gph} = \sum_{v,u} \res{(v,u)}$ decreases when adding new edges -- which offer a new interpretation as to why spatial methods help combat over-squashing. 

\paragraph{What about curvature?} Our analysis also sheds further light on the relation between over-squashing and curvature derived in \citet{topping2021understanding}. If the effective resistance is bounded from above, this leads to lower bounds for the resistance curvature introduced in \citet{devriendt2022discrete} and hence, under some assumptions, for the Ollivier curvature too \citep{ollivier2007ricci, ollivier2009ricci}. Our analysis then recovers why preventing the curvature from being `too' negative has benefits in terms of 
reducing over-squashing. 

\paragraph{About the Graph Transfer task.} We finally note that the results in \Cref{fig:graph-transfer-main} also validate the theoretical findings of \Cref{thm:effective_resistance}. If $v,u$ represent target and source nodes on the different graph-transfer topologies, then $\res(v,u)$ is highest for $\mathsf{CliquePath}$ and lowest for the $\mathsf{CrossedRing}$. Once again, the distance is only a partial information. Effective resistance provides a better picture for the impact of topology to over-squashing and hence the accuracy on the task; in \Cref{app:signal-prop} we further validate that via a synthetic experiment where we study how the propagation of a signal in a $\MPNN$ is affected by the effective resistance of $\gph$.

\begin{tcolorbox}[boxsep=0mm,left=2.5mm,right=2.5mm]
\textbf{Message of the Section:} \em  $\MPNN$s struggle to send information among nodes with high commute (access) time (equivalently, effective resistance). This connection between over-squashing and commute (access) time provides a unified framework for explaining why spatial and spectral-rewiring approaches manage to alleviate over-squashing.
\end{tcolorbox}
\vspace{-5pt}

\section{Conclusion and discussion}\label{sec:conclusion}

{\bf What did we do?} In this work, we investigated the role played by width, depth, and topology in the over-squashing phenomenon. We have proved that, while width can partly mitigate this problem, depth is, instead, generally bound to fail since over-squashing spills into vanishing gradients for a large number of layers. In fact, we have shown that the graph-topology plays the biggest role, with the commute (access) time providing a strong indicator for whether over-squashing is likely to happen independently of the number of layers. As a consequence of our analysis, we can draw a unified framework where we can rigorously justify why all recently proposed rewiring methods do alleviate over-squashing.

\paragraph{Limitations.} 
Strictly speaking, the analysis in this work applies to $\MPNN$s that weigh each edge contribution the same, up to a degree normalization. In the opposite case, which, for example, includes $\mathsf{GAT}$ \citep{velivckovic2017graph} and $\mathsf{GatedGCN}$ \citep{bresson2017residual}, over-squashing can be further mitigated 
by pruning the graph, 
hence alleviating the dispersion of information. However, the attention (gating) mechanism can fail if it is not able to identify which branches to ignore and can even amplify over-squashing by further reducing `useful' pathways. In fact, $\mathsf{GAT}$ still fails on the $\mathsf{Graph}$ $\mathsf{Transfer}$ task of \Cref{sec:depth}, albeit it seems to exhibit slightly more robustness. Extending the Jacobian bounds to this case is not hard, but will lead to less transparent formulas: a thorough analysis of this class, is left for future work. 
We also note that determining when the sensitivity is `too' small is generally also a function of the resolution of the readout, which we have not considered. Finally, \Cref{thm:effective_resistance} holds in expectation over the nonlinearity and, generally, \Cref{def:obst} encodes an average type of behaviour: a more refined (and exact) analysis is left for future work. 


{\bf Where to go from here.} We believe that the relation between over-squashing and vanishing gradient deserves further analysis. In particular, it seems that there is a phase transition that $\MPNN$s undergo from over-squashing of information between distant nodes, to vanishing of gradients at the level of the loss. In fact, this connection suggests that traditional methods that have been used in RNNs and GNNs to mitigate vanishing gradients, may also be beneficial for over-squashing. On a different note, this work has not touched on the important problem of over-smoothing; we believe that the theoretical connections we have derived, based on the relation between over-squashing, commute time, and Cheeger constant, suggest a much deeper interplay between these two phenomena. Finally, while this analysis confirms that both spatial and spectral-rewiring methods provably mitigate over-squashing, it does not tell us which method is preferable, when, and why. We hope that the theoretical investigation of over-squashing we have provided here, will also help tackle this important methodological question.

\section*{Acknowledgements}
We are grateful to Adrián Arnaiz, Johannes Lutzeyer, and Ismail Ceylan for providing insightful and detailed feedback and suggestions on an earlier version of the manuscript. We are also particularly thankful to Jacob Bamberger for helping us fix a technical assumption in one of our arguments. Finally, we are grateful to the anonymous reviewers for their input.  This research was supported in part by ERC Consolidator
grant No. 274228 (LEMAN) and by the EU and Innovation UK project TROPHY.

\clearpage

\bibliography{references}

\begin{thebibliography}{66}
\providecommand{\natexlab}[1]{#1}
\providecommand{\url}[1]{\texttt{#1}}
\expandafter\ifx\csname urlstyle\endcsname\relax
  \providecommand{\doi}[1]{doi: #1}\else
  \providecommand{\doi}{doi: \begingroup \urlstyle{rm}\Url}\fi

\bibitem[Abboud et~al.(2022)Abboud, Dimitrov, and Ceylan]{abboud2022shortest}
Abboud, R., Dimitrov, R., and Ceylan, I.~I.
\newblock Shortest path networks for graph property prediction.
\newblock In \emph{The First Learning on Graphs Conference}, 2022.
\newblock URL \url{https://openreview.net/forum?id=mWzWvMxuFg1}.

\bibitem[Abu-El-Haija et~al.(2019)Abu-El-Haija, Perozzi, Kapoor, Alipourfard,
  Lerman, Harutyunyan, Ver~Steeg, and Galstyan]{abu2019mixhop}
Abu-El-Haija, S., Perozzi, B., Kapoor, A., Alipourfard, N., Lerman, K.,
  Harutyunyan, H., Ver~Steeg, G., and Galstyan, A.
\newblock Mixhop: Higher-order graph convolutional architectures via sparsified
  neighborhood mixing.
\newblock In \emph{international conference on machine learning}, pp.\  21--29.
  PMLR, 2019.

\bibitem[Alon \& Yahav(2021)Alon and Yahav]{alon2020bottleneck}
Alon, U. and Yahav, E.
\newblock On the bottleneck of graph neural networks and its practical
  implications.
\newblock In \emph{International Conference on Learning Representations}, 2021.

\bibitem[Arnaiz-Rodr{\'i}guez et~al.(2022)Arnaiz-Rodr{\'i}guez, Begga,
  Escolano, and Oliver]{arnaiz2022diffwire}
Arnaiz-Rodr{\'i}guez, A., Begga, A., Escolano, F., and Oliver, N.
\newblock {DiffWire: Inductive Graph Rewiring via the Lov{\'a}sz Bound}.
\newblock In \emph{The First Learning on Graphs Conference}, 2022.
\newblock URL \url{https://openreview.net/pdf?id=IXvfIex0mX6f}.

\bibitem[Banerjee et~al.(2022)Banerjee, Karhadkar, Wang, Alon, and
  Mont{\'u}far]{banerjee2022oversquashing}
Banerjee, P.~K., Karhadkar, K., Wang, Y.~G., Alon, U., and Mont{\'u}far, G.
\newblock Oversquashing in gnns through the lens of information contraction and
  graph expansion.
\newblock In \emph{Annual Allerton Conference on Communication, Control, and
  Computing (Allerton)}, pp.\  1--8. IEEE, 2022.

\bibitem[Barcel{\'o} et~al.(2019)Barcel{\'o}, Kostylev, Monet, P{\'e}rez,
  Reutter, and Silva]{barcelo2019logical}
Barcel{\'o}, P., Kostylev, E.~V., Monet, M., P{\'e}rez, J., Reutter, J., and
  Silva, J.~P.
\newblock The logical expressiveness of graph neural networks.
\newblock In \emph{International Conference on Learning Representations}, 2019.

\bibitem[Bartlett et~al.(2017)Bartlett, Foster, and
  Telgarsky]{bartlett2017spectrally}
Bartlett, P.~L., Foster, D.~J., and Telgarsky, M.~J.
\newblock Spectrally-normalized margin bounds for neural networks.
\newblock \emph{Advances in neural information processing systems}, 30, 2017.

\bibitem[Bengio et~al.(1994)Bengio, Simard, and Frasconi]{bengio1994learning}
Bengio, Y., Simard, P., and Frasconi, P.
\newblock Learning long-term dependencies with gradient descent is difficult.
\newblock \emph{IEEE transactions on neural networks}, 5\penalty0 (2):\penalty0
  157--166, 1994.

\bibitem[Black et~al.(2023)Black, Nayyeri, Wan, and
  Wang]{black2023understanding}
Black, M., Nayyeri, A., Wan, Z., and Wang, Y.
\newblock Understanding oversquashing in gnns through the lens of effective
  resistance.
\newblock \emph{arXiv preprint arXiv:2302.06835}, 2023.

\bibitem[Bodnar et~al.(2021{\natexlab{a}})Bodnar, Frasca, Otter, Wang, Lio,
  Montufar, and Bronstein]{bodnar2021weisfeilercell}
Bodnar, C., Frasca, F., Otter, N., Wang, Y., Lio, P., Montufar, G.~F., and
  Bronstein, M.
\newblock Weisfeiler and lehman go cellular: Cw networks.
\newblock In \emph{Advances in Neural Information Processing Systems},
  volume~34, pp.\  2625--2640, 2021{\natexlab{a}}.

\bibitem[Bodnar et~al.(2021{\natexlab{b}})Bodnar, Frasca, Wang, Otter,
  Montúfar, Lió, and Bronstein]{bodnar2021weisfeiler}
Bodnar, C., Frasca, F., Wang, Y., Otter, N., Montúfar, G.~F., Lió, P., and
  Bronstein, M.~M.
\newblock Weisfeiler and lehman go topological: Message passing simplicial
  networks.
\newblock In \emph{International Conference on Machine Learning}, pp.\
  1026--1037, 2021{\natexlab{b}}.

\bibitem[Bodnar et~al.(2022)Bodnar, Giovanni, Chamberlain, Lió, and
  Bronstein]{bodnar2022neural}
Bodnar, C., Giovanni, F.~D., Chamberlain, B.~P., Lió, P., and Bronstein, M.~M.
\newblock Neural sheaf diffusion: A topological perspective on heterophily and
  oversmoothing in {GNN}s.
\newblock In \emph{Advances in Neural Information Processing Systems}, 2022.

\bibitem[Bresson \& Laurent(2017)Bresson and Laurent]{bresson2017residual}
Bresson, X. and Laurent, T.
\newblock Residual gated graph convnets.
\newblock \emph{arXiv preprint arXiv:1711.07553}, 2017.

\bibitem[Br{\"u}el-Gabrielsson et~al.(2022)Br{\"u}el-Gabrielsson, Yurochkin,
  and Solomon]{bruel2022rewiring}
Br{\"u}el-Gabrielsson, R., Yurochkin, M., and Solomon, J.
\newblock Rewiring with positional encodings for graph neural networks.
\newblock \emph{arXiv preprint arXiv:2201.12674}, 2022.

\bibitem[Bruna et~al.(2014)Bruna, Zaremba, Szlam, and LeCun]{bruna2013spectral}
Bruna, J., Zaremba, W., Szlam, A., and LeCun, Y.
\newblock Spectral networks and locally connected networks on graphs.
\newblock In \emph{International Conference on Learning Representations}, 2014.

\bibitem[Cai \& Wang(2020)Cai and Wang]{cai2020note}
Cai, C. and Wang, Y.
\newblock A note on over-smoothing for graph neural networks.
\newblock \emph{arXiv preprint arXiv:2006.13318}, 2020.

\bibitem[Chandra et~al.(1996)Chandra, Raghavan, Ruzzo, Smolensky, and
  Tiwari]{chandra1996electrical}
Chandra, A.~K., Raghavan, P., Ruzzo, W.~L., Smolensky, R., and Tiwari, P.
\newblock The electrical resistance of a graph captures its commute and cover
  times.
\newblock \emph{computational complexity}, 6\penalty0 (4):\penalty0 312--340,
  1996.

\bibitem[Chen et~al.(2020)Chen, Li, and Bruna]{chen2020supervised}
Chen, Z., Li, L., and Bruna, J.
\newblock Supervised community detection with line graph neural networks.
\newblock In \emph{International conference on learning representations}, 2020.

\bibitem[Chung \& Graham(1997)Chung and Graham]{chung1997spectral}
Chung, F.~R. and Graham, F.~C.
\newblock \emph{Spectral graph theory}.
\newblock American Mathematical Soc., 1997.

\bibitem[Dasoulas et~al.(2021)Dasoulas, Lutzeyer, and
  Vazirgiannis]{dasoulas2021learning}
Dasoulas, G., Lutzeyer, J.~F., and Vazirgiannis, M.
\newblock Learning parametrised graph shift operators.
\newblock In \emph{International Conference on Learning Representations}, 2021.

\bibitem[Deac et~al.(2022)Deac, Lackenby, and
  Veli{\v{c}}kovi{\'c}]{deac2022expander}
Deac, A., Lackenby, M., and Veli{\v{c}}kovi{\'c}, P.
\newblock Expander graph propagation.
\newblock In \emph{The First Learning on Graphs Conference}, 2022.

\bibitem[Defferrard et~al.(2016)Defferrard, Bresson, and
  Vandergheynst]{defferrard2016convolutional}
Defferrard, M., Bresson, X., and Vandergheynst, P.
\newblock Convolutional neural networks on graphs with fast localized spectral
  filtering.
\newblock In \emph{Advances in neural information processing systems},
  volume~29, 2016.

\bibitem[Devriendt \& Lambiotte(2022)Devriendt and
  Lambiotte]{devriendt2022discrete}
Devriendt, K. and Lambiotte, R.
\newblock Discrete curvature on graphs from the effective resistance.
\newblock \emph{Journal of Physics: Complexity}, 2022.

\bibitem[Di~Giovanni et~al.(2022{\natexlab{a}})Di~Giovanni, Luise, and
  Bronstein]{di2022heterogeneous}
Di~Giovanni, F., Luise, G., and Bronstein, M.
\newblock Heterogeneous manifolds for curvature-aware graph embedding.
\newblock In \emph{International Conference on Learning Representations
  Workshop on Geometrical and Topological Representation Learning},
  2022{\natexlab{a}}.

\bibitem[Di~Giovanni et~al.(2022{\natexlab{b}})Di~Giovanni, Rowbottom,
  Chamberlain, Markovich, and Bronstein]{di2022graph}
Di~Giovanni, F., Rowbottom, J., Chamberlain, B.~P., Markovich, T., and
  Bronstein, M.~M.
\newblock Graph neural networks as gradient flows.
\newblock \emph{arXiv preprint arXiv:2206.10991}, 2022{\natexlab{b}}.

\bibitem[D{\"o}rfler et~al.(2018)D{\"o}rfler, Simpson-Porco, and
  Bullo]{dorfler2018electrical}
D{\"o}rfler, F., Simpson-Porco, J.~W., and Bullo, F.
\newblock Electrical networks and algebraic graph theory: Models, properties,
  and applications.
\newblock \emph{Proceedings of the IEEE}, 106\penalty0 (5):\penalty0 977--1005,
  2018.

\bibitem[Ellens et~al.(2011)Ellens, Spieksma, Van~Mieghem, Jamakovic, and
  Kooij]{ellens2011effective}
Ellens, W., Spieksma, F.~M., Van~Mieghem, P., Jamakovic, A., and Kooij, R.~E.
\newblock Effective graph resistance.
\newblock \emph{Linear algebra and its applications}, 435\penalty0
  (10):\penalty0 2491--2506, 2011.

\bibitem[Gilmer et~al.(2017)Gilmer, Schoenholz, Riley, Vinyals, and
  Dahl]{gilmer2017neural}
Gilmer, J., Schoenholz, S.~S., Riley, P.~F., Vinyals, O., and Dahl, G.~E.
\newblock Neural message passing for quantum chemistry.
\newblock In \emph{International Conference on Machine Learning}, pp.\
  1263--1272. PMLR, 2017.

\bibitem[Goller \& Kuchler(1996)Goller and Kuchler]{goller1996learning}
Goller, C. and Kuchler, A.
\newblock Learning task-dependent distributed representations by
  backpropagation through structure.
\newblock In \emph{Proceedings of International Conference on Neural Networks
  (ICNN'96)}, volume~1, pp.\  347--352. IEEE, 1996.

\bibitem[Gori et~al.(2005)Gori, Monfardini, and Scarselli]{gori2005new}
Gori, M., Monfardini, G., and Scarselli, F.
\newblock A new model for learning in graph domains.
\newblock In \emph{Proceedings. 2005 IEEE International Joint Conference on
  Neural Networks, 2005.}, volume~2, pp.\  729--734. IEEE, 2005.

\bibitem[Gutteridge et~al.(2023)Gutteridge, Dong, Bronstein, and
  Di~Giovanni]{gutteridge2023drew}
Gutteridge, B., Dong, X., Bronstein, M., and Di~Giovanni, F.
\newblock Drew: Dynamically rewired message passing with delay.
\newblock \emph{arXiv preprint arXiv:2305.08018}, 2023.

\bibitem[Hamilton et~al.(2017)Hamilton, Ying, and
  Leskovec]{hamilton2017inductive}
Hamilton, W.~L., Ying, R., and Leskovec, J.
\newblock Inductive representation learning on large graphs.
\newblock In \emph{Advances in Neural Information Processing Systems}, pp.\
  1025--1035, 2017.

\bibitem[Hochreiter \& Schmidhuber(1997)Hochreiter and
  Schmidhuber]{hochreiter1997long}
Hochreiter, S. and Schmidhuber, J.
\newblock Long short-term memory.
\newblock \emph{Neural computation}, 9\penalty0 (8):\penalty0 1735--1780, 1997.

\bibitem[Jegelka(2022)]{jegelka2022theory}
Jegelka, S.
\newblock Theory of graph neural networks: Representation and learning.
\newblock \emph{arXiv preprint arXiv:2204.07697}, 2022.

\bibitem[Karhadkar et~al.(2022)Karhadkar, Banerjee, and
  Mont{\'u}far]{karhadkar2022fosr}
Karhadkar, K., Banerjee, P.~K., and Mont{\'u}far, G.
\newblock Fosr: First-order spectral rewiring for addressing oversquashing in
  gnns.
\newblock \emph{arXiv preprint arXiv:2210.11790}, 2022.

\bibitem[Kawaguchi(2016)]{kawaguchi2016deep}
Kawaguchi, K.
\newblock Deep learning without poor local minima.
\newblock In \emph{Advances in neural information processing systems},
  volume~29, 2016.

\bibitem[Kipf \& Welling(2017)Kipf and Welling]{kipf2016semi}
Kipf, T.~N. and Welling, M.
\newblock {Semi-Supervised Classification with Graph Convolutional Networks}.
\newblock In \emph{International Conference on Learning Representations}, 2017.

\bibitem[Klicpera et~al.(2019)Klicpera, Weissenberger, and
  G\"{u}nnemann]{klicpera2019diffusion}
Klicpera, J., Weissenberger, S., and G\"{u}nnemann, S.
\newblock Diffusion improves graph learning.
\newblock In \emph{Advances in Neural Information Processing Systems}, 2019.

\bibitem[Kreuzer et~al.(2021)Kreuzer, Beaini, Hamilton, L{\'e}tourneau, and
  Tossou]{kreuzer2021rethinking}
Kreuzer, D., Beaini, D., Hamilton, W., L{\'e}tourneau, V., and Tossou, P.
\newblock Rethinking graph transformers with spectral attention.
\newblock In \emph{Advances in Neural Information Processing Systems},
  volume~34, pp.\  21618--21629, 2021.

\bibitem[Li et~al.(2019)Li, Muller, Thabet, and Ghanem]{li2019deepgcns}
Li, G., Muller, M., Thabet, A., and Ghanem, B.
\newblock Deepgcns: Can gcns go as deep as cnns?
\newblock In \emph{Proceedings of the IEEE/CVF international conference on
  computer vision}, pp.\  9267--9276, 2019.

\bibitem[Li et~al.(2021)Li, M{\"u}ller, Ghanem, and Koltun]{li2021training}
Li, G., M{\"u}ller, M., Ghanem, B., and Koltun, V.
\newblock Training graph neural networks with 1000 layers.
\newblock In \emph{International conference on machine learning}, pp.\
  6437--6449. PMLR, 2021.

\bibitem[Lov{\'a}sz(1993)]{lovasz1993random}
Lov{\'a}sz, L.
\newblock Random walks on graphs.
\newblock \emph{Combinatorics, Paul erdos is eighty}, 2\penalty0
  (1-46):\penalty0 4, 1993.

\bibitem[Ma et~al.(2020)Ma, Xuan, Wang, Li, and Li{\`o}]{ma2020path}
Ma, Z., Xuan, J., Wang, Y.~G., Li, M., and Li{\`o}, P.
\newblock Path integral based convolution and pooling for graph neural
  networks.
\newblock In \emph{Advances in Neural Information Processing Systems},
  volume~33, pp.\  16421--16433, 2020.

\bibitem[Mialon et~al.(2021)Mialon, Chen, Selosse, and
  Mairal]{mialon2021graphit}
Mialon, G., Chen, D., Selosse, M., and Mairal, J.
\newblock Graphit: Encoding graph structure in transformers.
\newblock \emph{CoRR}, abs/2106.05667, 2021.

\bibitem[Morris et~al.(2019)Morris, Ritzert, Fey, Hamilton, Lenssen, Rattan,
  and Grohe]{morris2019weisfeiler}
Morris, C., Ritzert, M., Fey, M., Hamilton, W.~L., Lenssen, J.~E., Rattan, G.,
  and Grohe, M.
\newblock Weisfeiler and leman go neural: Higher-order graph neural networks.
\newblock In \emph{AAAI Conference on Artificial Intelligence}, pp.\
  4602--4609. {AAAI} Press, 2019.

\bibitem[Nikolentzos et~al.(2020)Nikolentzos, Dasoulas, and
  Vazirgiannis]{nikolentzos2020k}
Nikolentzos, G., Dasoulas, G., and Vazirgiannis, M.
\newblock k-hop graph neural networks.
\newblock \emph{Neural Networks}, 130:\penalty0 195--205, 2020.

\bibitem[Nt \& Maehara(2019)Nt and Maehara]{nt2019revisiting}
Nt, H. and Maehara, T.
\newblock Revisiting graph neural networks: All we have is low-pass filters.
\newblock \emph{arXiv preprint arXiv:1905.09550}, 2019.

\bibitem[Ollivier(2007)]{ollivier2007ricci}
Ollivier, Y.
\newblock Ricci curvature of metric spaces.
\newblock \emph{Comptes Rendus Mathematique}, 345\penalty0 (11):\penalty0
  643--646, 2007.

\bibitem[Ollivier(2009)]{ollivier2009ricci}
Ollivier, Y.
\newblock Ricci curvature of markov chains on metric spaces.
\newblock \emph{Journal of Functional Analysis}, 256\penalty0 (3):\penalty0
  810--864, 2009.

\bibitem[Pascanu et~al.(2013)Pascanu, Mikolov, and
  Bengio]{pascanu2013difficulty}
Pascanu, R., Mikolov, T., and Bengio, Y.
\newblock On the difficulty of training recurrent neural networks.
\newblock In \emph{International conference on machine learning}, pp.\
  1310--1318. PMLR, 2013.

\bibitem[Rampasek et~al.(2022)Rampasek, Galkin, Dwivedi, Luu, Wolf, and
  Beaini]{rampavsek2022recipe}
Rampasek, L., Galkin, M., Dwivedi, V.~P., Luu, A.~T., Wolf, G., and Beaini, D.
\newblock Recipe for a general, powerful, scalable graph transformer.
\newblock In \emph{Advances in Neural Information Processing Systems}, 2022.

\bibitem[Ruiz et~al.(2020)Ruiz, Gama, and Ribeiro]{ruiz2020gated}
Ruiz, L., Gama, F., and Ribeiro, A.
\newblock Gated graph recurrent neural networks.
\newblock \emph{IEEE Transactions on Signal Processing}, 68:\penalty0
  6303--6318, 2020.

\bibitem[Rusch \& Mishra(2021{\natexlab{a}})Rusch and Mishra]{rusch2021coupled}
Rusch, T.~K. and Mishra, S.
\newblock Coupled oscillatory recurrent neural network (co{\{}rnn{\}}): An
  accurate and (gradient) stable architecture for learning long time
  dependencies.
\newblock In \emph{International Conference on Learning Representations},
  2021{\natexlab{a}}.
\newblock URL \url{https://openreview.net/forum?id=F3s69XzWOia}.

\bibitem[Rusch \& Mishra(2021{\natexlab{b}})Rusch and
  Mishra]{rusch2021unicornn}
Rusch, T.~K. and Mishra, S.
\newblock Unicornn: A recurrent model for learning very long time dependencies.
\newblock In \emph{International Conference on Machine Learning}, pp.\
  9168--9178. PMLR, 2021{\natexlab{b}}.

\bibitem[Rusch et~al.(2022)Rusch, Chamberlain, Rowbottom, Mishra, and
  Bronstein]{rusch2022graph}
Rusch, T.~K., Chamberlain, B.~P., Rowbottom, J., Mishra, S., and Bronstein,
  M.~M.
\newblock Graph-coupled oscillator networks.
\newblock In \emph{International Conference on Machine Learning}, 2022.

\bibitem[Scarselli et~al.(2008)Scarselli, Gori, Tsoi, Hagenbuchner, and
  Monfardini]{scarselli2008graph}
Scarselli, F., Gori, M., Tsoi, A.~C., Hagenbuchner, M., and Monfardini, G.
\newblock The graph neural network model.
\newblock \emph{IEEE transactions on neural networks}, 20\penalty0
  (1):\penalty0 61--80, 2008.

\bibitem[Sperduti(1993)]{sperduti1994encoding}
Sperduti, A.
\newblock Encoding labeled graphs by labeling raam.
\newblock In \emph{Advances in Neural Information Processing Systems},
  volume~6, 1993.

\bibitem[Thomassen(1990)]{thomassen1990resistances}
Thomassen, C.
\newblock Resistances and currents in infinite electrical networks.
\newblock \emph{Journal of Combinatorial Theory, Series B}, 49\penalty0
  (1):\penalty0 87--102, 1990.

\bibitem[Topping et~al.(2022)Topping, Di~Giovanni, Chamberlain, Dong, and
  Bronstein]{topping2021understanding}
Topping, J., Di~Giovanni, F., Chamberlain, B.~P., Dong, X., and Bronstein,
  M.~M.
\newblock Understanding over-squashing and bottlenecks on graphs via curvature.
\newblock In \emph{International Conference on Learning Representations}, 2022.

\bibitem[Velingker et~al.(2022)Velingker, Sinop, Ktena, Veli{\v{c}}kovi{\'c},
  and Gollapudi]{velingker2022affinity}
Velingker, A., Sinop, A.~K., Ktena, I., Veli{\v{c}}kovi{\'c}, P., and
  Gollapudi, S.
\newblock Affinity-aware graph networks.
\newblock \emph{arXiv preprint arXiv:2206.11941}, 2022.

\bibitem[Veličković et~al.(2018)Veličković, Cucurull, Casanova, Romero,
  Liò, and Bengio]{velivckovic2017graph}
Veličković, P., Cucurull, G., Casanova, A., Romero, A., Liò, P., and Bengio,
  Y.
\newblock Graph attention networks.
\newblock In \emph{International Conference on Learning Representations}, 2018.

\bibitem[Wang et~al.(2020)Wang, Ying, Huang, and Leskovec]{wang2020multi}
Wang, G., Ying, R., Huang, J., and Leskovec, J.
\newblock Multi-hop attention graph neural network.
\newblock \emph{arXiv preprint arXiv:2009.14332}, 2020.

\bibitem[Xu et~al.(2018)Xu, Li, Tian, Sonobe, Kawarabayashi, and
  Jegelka]{xu2018representation}
Xu, K., Li, C., Tian, Y., Sonobe, T., Kawarabayashi, K.-i., and Jegelka, S.
\newblock Representation learning on graphs with jumping knowledge networks.
\newblock In \emph{International Conference on Machine Learning}, pp.\
  5453--5462. PMLR, 2018.

\bibitem[Xu et~al.(2019)Xu, Hu, Leskovec, and Jegelka]{xu2018how}
Xu, K., Hu, W., Leskovec, J., and Jegelka, S.
\newblock How powerful are graph neural networks?
\newblock In \emph{International Conference on Learning Representations}, 2019.

\bibitem[Ying et~al.(2021)Ying, Cai, Luo, Zheng, Ke, He, Shen, and
  Liu]{ying2021transformers}
Ying, C., Cai, T., Luo, S., Zheng, S., Ke, G., He, D., Shen, Y., and Liu, T.-Y.
\newblock Do transformers really perform badly for graph representation?
\newblock In \emph{Advances in Neural Information Processing Systems},
  volume~34, pp.\  28877--28888, 2021.

\bibitem[Zhao et~al.(2022)Zhao, Wang, Han, and Guo]{zhao2022analysis}
Zhao, W., Wang, C., Han, C., and Guo, T.
\newblock Analysis of graph neural networks with theory of markov chains.
\newblock 2022.

\end{thebibliography}
\bibliographystyle{icml2023}

\newpage
\appendix
\onecolumn

\section{General preliminaries}\label{app:sec_preliminaries}
We first introduce important quantities and notations used throughout our proofs. We take a graph $\gph$ with nodes $ \mathsf{V}$ and edges $\mathsf{E}\subset \V \times \V$, to be simple, undirected, and connected. We let $n = \lvert \V \rvert$ and 
write $[n]:= \{1,\ldots,n\}$. We denote the adjacency matrix by $\mathbf{A}\in\R^{n\times n}$. We compute the degree of $v\in \V$ by $d_v = \sum_{u}A_{vu}$ and write 
$\mathbf{D} = \mathrm{diag}(d_1,\ldots,d_n)$. One can take different normalizations of $\mathbf{A}$, so we write $\Anorm\in\R^{n\times n}$ for a Graph Shift Operator (GSO), i.e an $n\times n$ matrix satisfying $\Anorm_{vu} \neq 0$ if and only if $(v,u)\in \mathsf{E}$; typically, we have $\Anorm \in \{\mathbf{A},\mathbf{D}^{-1}\mathbf{A},\mathbf{D}^{-1/2}\mathbf{A}\mathbf{D}^{-1/2}\}$. Finally, $d_{\gph}(v,u)$ is the {\bf shortest walk} ({\bf geodesic}) distance between nodes $v$ and $u$. 

\paragraph{Graph spectral properties: the eigenvalues.} The (normalized) graph Laplacian is defined as $\DELta = \mathbf{I} - \mathbf{D}^{-\frac{1}{2}}\mathbf{A}\mathbf{D}^{-\frac{1}{2}}$. This is a symmetric, positive semi-definite operator on $\gph$. Its eigenvalues can be ordered as $\lambda_0 < \lambda_1 \leq \ldots \leq \lambda_{n-1}$. The smallest eigenvalue $\lambda_0$ is always zero, with multiplicity given by the number of connected components of $\gph$ \citep{chung1997spectral}. Conversely, the largest eigenvalue $\lambda_{n-1}$ is always strictly smaller than $2$ whenever the graph is not bipartite. Finally, we recall that the smallest, positive, eigenvalue $\lambda_1$ is known as the {\bf spectral gap}. In several of our proofs we rely on this quantity to provide convergence rates. We also recall that the spectral gap is related to the Cheeger constant -- introduced in Definition \ref{def:cheeger} -- of $\gph$ via the Cheeger inequality:
\begin{equation}\label{eq:cheeger_inequality}
    2\cheeg \geq \lambda_{1} > \frac{\cheeg^{2}}{2}.
\end{equation}

\paragraph{Graph spectral properties: the eigenvectors.}
Throughout the appendix, we let $\{\eigen_{\ell}\}$ be a family of orthonormal eigenvectors of $\DELta$. In particular, we note that the eigenspace associated with $\lambda_0$ represents the space of signals that respect the graph topology the most (i.e. the smoothest signals), so that we can write $(\eigen_0)_v = \sqrt{d_v}/2\lvert\E\rvert$, for any $v\in\V$.

As common, from now on we assume that the graph is {\em not} bipartite, so that $\lambda_{n-1} < 2$. 
We let $\mathbf{H}^{(0)}\in\R^{n\times p}$ be the matrix representation of node {\em features}, with $p$ denoting the hidden dimension. Features of node $v$ produced by layer $t$ of an $\MPNN$ are denoted by $\mathbf{h}_{v}^{(t)}$ and we write their components as $(\mathbf{h}_{v}^{(t)})^{\alpha} := h_{v}^{(t),\alpha}$, for $\alpha \in [p]$.

\paragraph{Einstein summation convention.} To ease notations when deriving the bounds on the Jacobian, in the proof below we often rely on Einstein summation convention, meaning that, unless specified otherwise, we always sum across repeated indices: for example, when we write terms like $x_{\alpha}y^{\alpha}$, we are tacitly omitting the symbol $\sum_{\alpha}$.

\section{Proofs of \Cref{sec:width}}\label{app:sec_width}

In this Section we demonstrate the results in \Cref{sec:width}. In fact, we derive a sensitivity bound far more general than \Cref{cor:bound_MLP_MPNN} that, in particular, extends to $\MPNN$s that can stack multiple layers (MLPs) in the aggregation phase. We introduce a class of $\MPNN$s of the form:
\begin{equation}\label{eq:isotropic_MPNN}
    \mathbf{h}_{v}^{(t)} = \mathsf{up}^{(t)}\Big(\mathsf{rs}^{(t)}(\mathbf{h}_{v}^{(t-1)}) + \mathsf{mp}^{(t)}\Big(\sum_{u}\Anorm_{vu}\mathbf{h}_{u}^{(t-1)}\Big) \Big)
\end{equation}
\noindent for learnable update, residual, and message-passing maps $\mathsf{up}^{(t)},\mathsf{rs}^{(t)},\mathsf{mp}^{(t)}:\R^{p}\rightarrow \R^{p}$. Note that Eq.~\eqref{eq:isotropic_MPNN} includes common $\MPNN$s like $\mathsf{GCN}$ \citep{kipf2016semi}, $\mathsf{SAGE}$ \citep{hamilton2017inductive}, and $\mathsf{GIN}$ \citep{xu2018how}, where $\Anorm$ is $\mathbf{D}^{-1/2}\mathbf{A}\mathbf{D}^{-1/2}$, $\mathbf{D}^{-1}\mathbf{A}$ and $\mathbf{A}$, respectively. An $\MPNN$ usually has Lipschitz maps, with Lipschitz constants typically depending on regularization of the weights to promote generalization. We say that an $\MPNN$ as in Eq.~\eqref{eq:isotropic_MPNN} is $(c_{\mathsf{up}},c_{\mathsf{rs}},c_{\mathsf{mp}})$-regular, if for $t\in[m]$ and $\alpha\in[p]$, we have
\begin{align*}
    \| \nabla (\mathsf{up}^{(t)})^{\alpha}\|_{L_{1}} \leq c_{\mathsf{up}}, \quad \| \nabla (\mathsf{rs}^{(t)})^{\alpha}\|_{L_{1}} \leq c_{\mathsf{rs}}, \quad
    \| \nabla (\mathsf{mp}^{(t)})^{\alpha}\|_{L_{1}} \leq  c_{\mathsf{mp}}.
\end{align*}
\noindent As in \citet{xu2018representation, topping2021understanding}, we study the propagation of information in the $\MPNN$ via the Jacobian of node features after $m$ layers. A small derivative of $\mathbf{h}_v^{(m)}$ with respect to $\mathbf{h}_u^{(0)}$ means that -- {\bf at the first-order} -- the representation at node $v$ is mostly insensitive to the information contained at $u$ (e.g. its atom type, if $\gph$ is a molecule).
\begin{theorem}\label{thm:bound_general} Given a $(c_{\mathsf{up}},c_{\mathsf{rs}},c_{\mathsf{mp}})$-regular $\MPNN$ for $m$ layers  and nodes $v,u\in\mathsf{V}$, we have
\begin{equation}
   \left\|\frac{\partial \mathbf{h}_{v}^{(m)}}{\partial \mathbf{h}_{u}^{(0)}}\right\|_{L_1} \leq p\cdot c^{m}_{\mathsf{up}}\left(\left(c_{\mathsf{rs}}\mathbf{I} +  c_{\mathsf{mp}}\Anorm\right)^{m}\right)_{vu}.
\end{equation}
\end{theorem}

\begin{proof}
We prove the result above by induction on the number of layers $m$. We fix $\alpha,\beta \in [p]$. In the case of $m=1$, we get (omitting to write the arguments where we evaluate the maps, and {\bf using the Einstein summation convention over repeated indices}):
\begin{equation*}
    \Big| \frac{\partial h_v^{(1),\alpha}}{\partial h_u^{(0),\beta}}\Big| = \Big| \partial_{p}\mathsf{up}^{(0),\alpha}\Big(\partial_r\mathsf{rs}^{(0),p}\frac{\partial h_v^{(0),r}}{\partial h_u^{(0),\beta}} + \partial_q\mathsf{mp}^{(0),p}\Anorm_{vz}\frac{\partial h_z^{(0),q}}{\partial h_u^{(0),\beta}}\Big) \Big|,
\end{equation*}
\noindent which can be readily reduced to
\begin{equation*}
     \Big| \frac{\partial h_v^{(1),\alpha}}{\partial h_u^{(0),\beta}}\Big| = \Big| \partial_{p}\mathsf{up}^{(0),\alpha}\Big(\partial_\beta\mathsf{rs}^{(0),p}\delta_{vu} + \partial_\beta\mathsf{mp}^{(0),p}\Anorm_{vu}\Big) \Big| \leq c_{\mathsf{up}}\left(c_{\mathsf{rs}}\mathbf{I}+ c_{\mathsf{mp}}\Anorm\right)_{vu},
\end{equation*}
\noindent thanks to the Lipschitz bounds on the $\MPNN$, which confirms the case of a single layer (i.e. $ m =1$). We now assume the bound to be satisfied for $m$ layers and use induction to derive
\begin{align*}
    \Big| \frac{\partial h_v^{(m+1),\alpha}}{\partial h_u^{(0),\beta}}\Big| &= \Big| \partial_{p}\mathsf{up}^{(m),\alpha}\Big(\partial_r\mathsf{rs}^{(m),p}\frac{\partial h_v^{(m),r}}{\partial h_u^{(0),\beta}} + \partial_q\mathsf{mp}^{(m),p}\Anorm_{vz}\frac{\partial h_z^{(m),q}}{\partial h_u^{(0),\beta}}\Big) \Big| \\
    &\leq \Big|\partial_{p}\mathsf{up}^{(m),\alpha}\Big|\Big(\left\vert \partial_r\mathsf{rs}^{(m),p}\right\vert \left(c_{\mathsf{up}}^{m}\left(\left(c_{\mathsf{rs}}\mathbf{I} + c_{\mathsf{mp}}\Anorm\right)^{m}\right)_{vu}\right) + \Big|\partial_q\mathsf{mp}^{(m),p} \Big| \Anorm_{vz}\left(c_{\mathsf{up}}^{m}\left(\left(c_{\mathsf{rs}}\mathbf{I} + c_{\mathsf{mp}}\Anorm\right)^{m}\right)_{zu}\right) \Big) \\
    &\leq \Big|\partial_{p}\mathsf{up}^{(m),\alpha}\Big|\Big(c_{\mathsf{rs}} \left(c_{\mathsf{up}}^{m}\left(\left(c_{\mathsf{rs}}\mathbf{I} + c_{\mathsf{mp}}\Anorm\right)^{m}\right)_{vu}\right) + c_{\mathsf{mp}} \Anorm_{vz}\left(c_{\mathsf{up}}^{m}\left(\left(c_{\mathsf{rs}}\mathbf{I} + c_{\mathsf{mp}}\Anorm\right)^{m}\right)_{zu}\right) \Big) \\
    &\leq c_{\mathsf{up}}^{m+1}\left(c_{\mathsf{rs}}\left(\left(c_{\mathsf{rs}}\mathbf{I} + c_{\mathsf{mp}}\Anorm\right)^{m}\right)_{vu} + c_{\mathsf{mp}}\Anorm_{vz}\left(\left(c_{\mathsf{rs}}\mathbf{I} + c_{\mathsf{mp}}\Anorm\right)^{m}\right)_{vu}\right) \\
    & = c_{\mathsf{up}}^{m+1}\Big(\left(c_{\mathsf{rs}}\mathbf{I} + c_{\mathsf{mp}}\Anorm\right)^{m+1}\Big)_{vu},
\end{align*}
\noindent where we have again using the Lipschitz bounds on the maps $\mathsf{up},\mathsf{rs},\mathsf{mp}$. This completes the induction argument.
\end{proof}

\noindent From now on we focus on the class of $\MPNN$ adopted in the main document, whose layer we report below for convenience:
\begin{equation*}
    \mathbf{h}_{v}^{(t+1)} = \up\Big(c_{\rs}\W_{\rs}^{(t)} \mathbf{h}_{v}^{(t)} + c_{\mpas}\W_{\mpas}^{(t)}\sum_{u}\Anorm_{vu}\mathbf{h}^{(t)}_{u}\Big),
\end{equation*}
\noindent We can adapt easily the general argument to derive \Cref{cor:bound_MLP_MPNN}.

\begin{proof}[Proof of \Cref{cor:bound_MLP_MPNN}] One can follow the steps in the proof of \Cref{thm:bound_general} and, again, proceed by induction. The case $m=1$ is straightforward, so we move to the inductive step and assume the bound to hold for $m$ arbitrary. Given $\alpha,\beta \in [p]$, we have
\begin{align*}
 \Big| \frac{\partial h_v^{(m+1),\alpha}}{\partial h_u^{(0),\beta}}\Big| &\leq \lvert \up' \rvert \Big( c_{\rs}\left\vert(\W_{\rs})^{(m)}_{\alpha\gamma}\right\vert\Big|\frac{\partial h_v^{(m),\gamma}}{\partial h_u^{(0),\beta}}\Big| + c_{\mpas}\left\vert (\W_{\mpas})^{(m)}_{\alpha\gamma}\right\vert \Anorm_{vz} \Big| \frac{\partial h_z^{(m),\gamma}}{\partial h_u^{(0),\beta}}\Big|\Big) \\
 &\leq c_{\up}w\left(c_{\rs}\left\|  \frac{\partial \mathbf{h}_v^{(m)}}{\partial \mathbf{h}_u^{(0)}}\right\|_{L_1} + c_{\mpas}\Anorm_{vz}\left\|  \frac{\partial \mathbf{h}_z^{(m)}}{\partial \mathbf{h}_u^{(0)}}\right\|_{L_1}  \right) \\
 &\leq c_{\sigma}w\left(c_{\sigma}wp\right)^{m}\left( c_{\rs}\left(\left(c_{\rs}\mathbf{I} + c_{\mpas}\Anorm\right)^{m}\right)_{vu} + c_{\mpas}\Anorm_{vz}\left(\left(c_{\rs}\mathbf{I} + c_{\mpas}\Anorm\right)^{m}\right)_{zu}\right) \\
 &\leq c_{\sigma}w\left(c_{\sigma}wp\right)^{m} \Big(\left(c_{\rs}\mathbf{I} + c_{\mpas}\Anorm\right)^{m+1}\Big)_{vu}.
\end{align*}
\noindent We can finally sum over $\alpha$ on the left and conclude the proof (this will generate an extra $p$ factor on the right hand side).

\end{proof}
\section{Proofs of \Cref{sec:depth}}\label{app:sec_depth}

{\bf Convention:} From now on we always let $\Anorm = \mathbf{D}^{-1/2}\mathbf{A}\mathbf{D}^{-1/2}$. The bounds in this Section extend easily to $\mathbf{D}^{-1}\mathbf{A}$ in light of the similarity of the two matrices since $\Anorm^{k} = \mathbf{D}^{1/2}\left(\mathbf{D}^{-1}\mathbf{A}\right)^{k}\mathbf{D}^{-1/2}$. For the unnormalized matrix $\mathbf{A}$ instead, things are slightly more subtle. In principle, this matrix is not normalized, and in fact, the entry $(\mathbf{A}^{k})_{vu}$ coincides with the number of walks from $v$ to $u$ of length $k$. In general, this will not lead to bounds decaying exponentially with the distance. However, if we go in expectation over the computational graph as in \citet{xu2018representation}, Appendix A of \citet{topping2021understanding} and \Cref{sec:topology}, one finds that nodes at smaller distance will still have sensitivity exponentially larger than nodes at large distance. This is also confirmed by our $\mathsf{Graph}$ $\mathsf{Transfer}$ synthetic experiments, where $\mathsf{GIN}$ struggles with long-range dependencies (in fact, even slightly more than $\mathsf{GCN}$, which uses the symmetrically normalized adjacency $\Anorm$).

We prove a sharper bound for Eq.~\eqref{eq:cor_distance} which contains \Cref{cor:over-squasing_distance} as a particular case.

\begin{theorem}
Given an $\MPNN$ as in Eq.~\eqref{eq:MPNN_mlp}, let $v,u\in\mathsf{V}$ be at distance $r$. Let $c_{\up}$ be the Lipschitz constant of $\sigma$, $w$ the maximal entry-value over all weight matrices, $d_{\mathrm{min}}$ be the minimal degree, and $\gamma_{\ell}(v,u)$ be the number of walks from $v$ to $u$ of maximal length $\ell$. For any $0 \leq k < r$, we have 

\begin{equation}
   \left\| \frac{\partial \mathbf{h}_{v}^{(r+k)}}{\partial \mathbf{h}_{u}^{(0)}}\right\|_{L_1} \leq \gamma_{r+k}(v,u) (c_{\up}(c_{\rs} + c_{\mpas})w p(k+1))^k \Big(\frac{2c_{\up}wpc_{\mpas}}{d_{\mathrm{min}}}\Big)^r.
\end{equation}
\end{theorem}
\begin{proof}
    Fix $v,u\in\V$ as in the statement and let $0 \leq k < r$. We can use the sensitivity bounds in \Cref{cor:bound_MLP_MPNN} and write that 
    \begin{equation*}
         \left\| \frac{\partial \mathbf{h}_{v}^{(r+k)}}{\partial \mathbf{h}_{u}^{(0)}}\right\|_{L_1} \leq \left(c_{\up}wp\right)^{r+k}\Big(\left(c_{\rs}\mathbf{I} + c_{\mpas}\Anorm\right)^{r+k}\Big)_{vu} = \left(c_{\up}wp\right)^{r+k}\sum_{i = 0}^{r+k}\binom{r+k}{i}c_{\rs}^{r+k - i}c_{\mpas}^{i}(\Anorm^{i})_{vu}.
    \end{equation*}
\noindent Since nodes $v,u$ are at distance $r$, the first $r$ terms of the sum above vanish. Once we recall that $\Anorm = \mathbf{D}^{-1/2}\mathbf{A}\mathbf{D}^{-1/2}$, we can bound the polynomial in the previous equation by
\begin{align*}
    \sum_{i = 0}^{r+k}\binom{r+k}{i}c_{\rs}^{r+k - i}c_{\mpas}^{i}(\Anorm^{i})_{vu} &= \sum_{i = r}^{r+k}\binom{r+k}{i}c_{\rs}^{r+k - i}c_{\mpas}^{i}(\Anorm^{i})_{vu} \leq \gamma_{r+k}(v,u)\sum_{i = r}^{r+k}\binom{r+k}{i}c_{\rs}^{r+k - i}\left(\frac{c_{\mpas}}{d_{\mathrm{min}}}\right)^{i} \\
    &= \gamma_{r+k}(v,u) \sum_{q = 0}^{k}\binom{r+k}{r+q}c_{\rs}^{k - q}\left(\frac{c_{\mpas}}{d_{\mathrm{min}}}\right)^{r + q} \\ 
    &= \gamma_{r+k}(v,u)\left(\frac{c_{\mpas}}{d_{\mathrm{min}}}\right)^{r} \sum_{q = 0}^{k}\binom{r+k}{r+q}c_{\rs}^{k - q}\left(\frac{c_{\mpas}}{d_{\mathrm{min}}}\right)^{q}. 
\end{align*}
\noindent We can now provide a simple estimate for 
\begin{align*}
    \binom{r+k}{r+q} &= \frac{(r+k)(r-1+k)\cdots(1+k)}{(r+q)(r-1+q)\cdots(1+q)}\binom{k}{q} \leq \frac{(r+k)(r-1+k)\cdots(1+k)}{r!}\binom{k}{q}  \\
    &\leq \left(1 + \frac{k}{r}\right)\cdots (1+k) \binom{k}{q} \leq \left(1+\frac{k}{k+1}\right)^{r-k}(1+k)^{k}\binom{k}{q}. \\
\end{align*}
\noindent Accordingly, the polynomial above can be expanded as
\begin{align*}
 \sum_{i = 0}^{r+k}\binom{r+k}{i}c_{\rs}^{r+k - i}c_{\mpas}^{i}(\Anorm^{i})_{vu} &\leq \gamma_{r+k}(v,u)\left(1+\frac{k}{k+1}\right)^{r-k}(1+k)^{k}\left(\frac{c_{\mpas}}{d_{\mathrm{min}}}\right)^{r} \sum_{q = 0}^{k}\binom{k}{q}c_{\rs}^{k - q}\left(\frac{c_{\mpas}}{d_{\mathrm{min}}}\right)^{q} \\
 &=  \gamma_{r+k}(v,u)\left(\frac{(1+k)^{2}}{2k+1}\left(c_{\rs} + \frac{c_{\mpas}}{d_{\mathrm{min}}}\right)\right)^{k}\left(\left(1+\frac{k}{k+1}\right)\frac{c_{\mpas}}{d_{\mathrm{min}}}\right)^{r} \\
 &\leq \gamma_{r+k}(v,u)\left(\frac{(1+k)^{2}}{2k+1}\left(c_{\rs} + \frac{c_{\mpas}}{d_{\mathrm{min}}}\right)\right)^{k}\left(\frac{2c_{\mpas}}{d_{\mathrm{min}}}\right)^{r}.
 \end{align*}
 \noindent We can then put all the ingredients together, and write the bound
 \begin{align*}
     \left| \frac{\partial \mathbf{h}_{v}^{(r+k)}}{\partial \mathbf{h}_{u}^{(0)}}\right\|_{L_1} &\leq\gamma_{r+k}(v,u) \left(c_{\up}wp\right)^{r+k}\left(\frac{(1+k)^{2}}{2k+1}\left(c_{\rs} + \frac{c_{\mpas}}{d_{\mathrm{min}}}\right)\right)^{k}\left(\frac{2c_{\mpas}}{d_{\mathrm{min}}}\right)^{r} \\
     & = \gamma_{r+k}(v,u)\left(c_{\up}wp\,\frac{(1+k)^{2}}{2k+1}\left(c_{\rs} + \frac{c_{\mpas}}{d_{\mathrm{min}}}\right)\right)^{k}\left(\frac{2c_{\up}wpc_{\mpas}}{d_{\mathrm{min}}}\right)^{r} \\
     &\leq \gamma_{r+k}(v,u)\left(c_{\up}\left(c_{\rs} + c_{\mpas}\right)wp(1+k)\right)^{k}\left(\frac{2c_{\up}wpc_{\mpas}}{d_{\mathrm{min}}}\right)^{r},
 \end{align*}
 \noindent which completes the proof. We finally note that this also proves \Cref{cor:over-squasing_distance}.
\end{proof}

\subsection{Vanishing gradients result}

We now report and demonstrate a more explicit version of \Cref{thm:vanishing}.

\begin{theorem}[\textbf{Vanishing gradients}]\label{thm:main_vanishing_app} Consider an $\MPNN$ as in Eq.~\eqref{eq:MPNN_mlp} for $m$ layers with a quadratic loss $\mathcal{L}$. Assume that (i) $\sigma$ has Lipschitz constant $c_{\up}$ and $\sigma(0) = 0$, and (ii) that all weight matrices have spectral norm bounded by $\mu > 0$. 
Given any weight $\theta$ entering a layer $k$, there exists a constant $C > 0$ independent of $m$, such that
\begin{align}
    \left\vert \frac{\partial \mathcal{L}}{\partial \theta}\right\vert &\leq C\left(c_{\up}\mu(c_{\rs} + c_{\mpas})\right)^{m-k}\left(1  +\left(c_{\up}\mu(c_{\rs} + c_{\mpas})\right)^{m}\right),
\end{align}
\noindent where $\lvert\lvert \mathbf{H}^{(0)}\rvert\rvert_{F}$ is the Frobenius norm of the input node features.
\end{theorem}
\begin{proof}
Consider a quadratic loss $\mathcal{L}$ of the form 
\begin{equation*}
    \mathcal{L}(\mathbf{H}^{(m)}) = \frac{1}{2}\sum_{v\in\V}\| \mathbf{h}_v^{(m)} - \mathbf{y}_v\|^{2},
\end{equation*}
\noindent and we let $\mathbf{Y}$ represent the node ground-truth values. Given a weight $\theta$ entering layer $k < m$, we can write the gradient of the loss as
\begin{equation*}
    \Big| \frac{\partial \mathcal{L}(\mathbf{H}^{(m)})}{\partial\theta}\Big| = \Big| \sum_{v,u\in\V}\sum_{\alpha,\beta\in[p]}\frac{\partial \mathcal{L}}{\partial h_v^{(m),\alpha}}\frac{\partial h_v^{(m),\alpha}}{\partial h_u^{(k),\beta}}\frac{\partial h_u^{(k),\beta}}{\partial \theta}\Big|.
\end{equation*}
\noindent Once we fix $k$, the term $\lvert \partial h_u^{(k),\beta} / \partial \theta \rvert$ is independent of $m$ and we can bound it by some constant $C$. Since we have a quadratic loss, to bound $\partial \mathcal{L} / \partial h_v^{(m),\alpha}$, it suffices to bound the solution of the $\MPNN$ after $m$ layers. First, we use the Kronecker product formalism to rewrite the $\MPNN$-update  in matricial form as 
\begin{equation}
    \mathbf{H}^{(m)} = \up\left(\left(c_{\rs}\OMEga^{(m)}\otimes \mathbf{I} + c_{\mpas}\W^{(m)}\otimes \Anorm\right)\mathbf{H}^{(m-1)}\right).
\end{equation}
\noindent Thanks to the Lipschitzness of $\up$ and the requirement $\up(0) = 0$, we derive
\begin{equation*}
  \| \mathbf{H}^{(m)} \|_{F} \leq c_{\up}\| c_{\rs}\OMEga^{(m)}\otimes \mathbf{I} + c_{\mpas}\W^{(m)}\otimes \Anorm\|_{2} \|\mathbf{H}^{(m-1)}\|_{F},
\end{equation*}
\noindent where $F$ indicates the Frobenius norm. Since the largest singular value of $\mathbf{B}\otimes \mathbf{C}$ is bounded by the product of the largest singular values, we deduce that -- recall that the largest eigenvalue of $\Anorm = \mathbf{D}^{-1/2}\mathbf{A}\mathbf{D}^{-1/2}$ is $1$:
\begin{equation}\label{app:eq:norm_bounded}
    \| \mathbf{H}^{(m)} \|_{F} \leq c_{\up}\mu(c_{\rs} + c_{\mpas})\|\mathbf{H}^{(m-1)}\|_{F} \leq (c_{\up}\mu(c_{\rs} + c_{\mpas}))^{m}\|\mathbf{H}^{(0)}\|_{F},
\end{equation}
\noindent which affords a control of the gradient of the loss w.r.t. the solution at the final layer being the loss quadratic. We then find
\begin{align}
     \Big| \frac{\partial \mathcal{L}(\mathbf{H}^{(m)})}{\partial\theta}\Big| &\leq C\Big| \sum_{v,u\in\V}\sum_{\alpha,\beta\in[p]}\frac{\partial \mathcal{L}}{\partial h_v^{(m),\alpha}}\frac{\partial h_v^{(m),\alpha}}{\partial h_u^{(k),\beta}}\Big| \notag \\ 
     &\leq C \sum_{v,u\in\V}\sum_{\beta\in[p]}\left\|\frac{\partial \mathcal{L}}{\partial \mathbf{h}_v^{(m)}}\right\|\left\|\frac{\partial \mathbf{h}_v^{(m)}}{\partial h_u^{(k),\beta}}\right\| \notag \\
     &\leq C \sum_{v,u\in\V}\sum_{\beta\in[p]}\Big(\|\mathbf{H}^{(m)}\|_{F} + \|\mathbf{Y}\|_{F}\Big)\left\|\frac{\partial \mathbf{h}_v^{(m)}}{\partial h_u^{(k),\beta}}\right\| \notag \\
     &\leq C \sum_{v,u\in\V}\sum_{\beta\in[p]}\Big((c_{\up}\mu(c_{\rs} + c_{\mpas}))^{m}\|\mathbf{H}^{(0)}\|_{F} + \|\mathbf{Y}\|_{F}\Big)\left\|\frac{\partial \mathbf{h}_v^{(m)}}{\partial h_u^{(k),\beta}}\right\| \label{eq:app:proof:vanishing:1}
\end{align}
\noindent where in the last step we have used Eq.~\eqref{app:eq:norm_bounded}. We now provide a {\bf new} bound on the sensitivity -- {\em differently from the analysis in earlier Sections, here we no longer account for the topological information depending on the choice of $v,u$ given that we need to integrate over all possible pairwise contributions to compute the gradient of the loss.} The idea below, is to apply the Kronecker product formalism to derive a single operator in the tensor product of feature and graph space acting on the Jacobian matrix -- this allows us to derive much sharper bounds. Note that, once we fix a node $u$ and a $\beta\in[p]$, we can write
\begin{align*}
\left\|\frac{\partial \mathbf{H}^{(m)}}{\partial h_u^{(k),\beta}}\right\|^{2} &\leq \sum_{v\in\V}\sum_{\alpha\in [p]}c_{\up}^{2}\Big( c_{\rs}\OMEga^{(m)}_{\alpha\gamma}\frac{\partial h_v^{(m-1),\gamma}}{\partial h_{u}^{(k),\beta}} + c_{\mpas}\W^{(m)}_{\alpha\gamma}\Anorm_{vz}\frac{\partial h_z^{(m-1),\gamma}}{\partial h_{u}^{(k),\beta}}  \Big)^{2} \\
&= c_{\up}^{2}\sum_{v\in\V}\sum_{\alpha\in [p]}\Big(\Big(c_{\rs}\OMEga^{(m)}\otimes \mathbf{I} + c_{\mpas}\W^{(m)}\otimes \Anorm \Big)\frac{\partial \mathbf{H}^{(m-1)}}{\partial h_u^{(k),\beta}}\Big)_{v,\alpha}^{2} \\
&\leq c_{\up}^{2}\| c_{\rs}\OMEga^{(m)}\otimes \mathbf{I} + c_{\mpas}\W^{(m)}\otimes \Anorm \|_{2}^{2} \left\|\frac{\partial \mathbf{H}^{(m-1)}}{\partial h_u^{(k),\beta}}\right\|_{F}^{2} 
\end{align*}
\noindent meaning that 
\begin{equation*}
    \left\|\frac{\partial \mathbf{H}^{(m)}}{\partial h_u^{(k),\beta}}\right\| \leq (c_{\up}\mu(c_{\rs} + c_{\mpas}))^{m-k},
\end{equation*}
\noindent where we have used that (i) the largest singular value of the weight matrices is $\mu$, (ii) that the largest eigenvalue of $c_{\rs}\mathbf{I} + c_{\mpas}\Anorm$ is $c_{\rs} + c_{\mpas}$ (as follows from $\Anorm = \mathbf{I} - \DELta$, and the spectral analysis of $\DELta$), (iii) that $\| \partial\mathbf{H}^{(k)} / \partial h_u^{(k),\beta} \| = 1$. Once we absorb the term $\| \mathbf{Y}\|$ in the constant $C$ in \eqref{eq:app:proof:vanishing:1}, we conclude the proof.

\end{proof}

\section{Proofs of \Cref{sec:topology}}\label{app:sec_topology}
\noindent In this Section we consider the convolutional family of $\MPNN$ in Eq.~\eqref{eq:mpnn_simplified}.
Before we prove the main results of this Section, we comment on the main assumption on the nonlinearity and formulate it more explicitly. Let us take $k < m$. When we compute the sensitivity of $\mathbf{h}_v^{(m)}$ to $\mathbf{h}_u^{(k)}$, we obtain a sum of different terms over all possible paths from $v$ to $u$ of length $m-k$. In this case, the derivative of 
$\mathsf{ReLU}$ acts as a Bernoulli variable evaluated along all these possible paths. Similarly to \citet{kawaguchi2016deep,xu2018representation}, we require the following:


\begin{assumption}\label{assumption} Assume that all paths in the computation graph
of the model are activated with the same probability of
success $\rho$. 
When we take the expectation $\mathbb{E}[\partial \mathbf{h}_{v}^{(m)} / \partial \mathbf{h}_{u}^{(k)}]$, we mean that we are taking the average over such Bernoulli variables. 
\end{assumption}

\noindent Thanks to \Cref{assumption}, we can follow the very same argument in the proof of Theorem 1 in \cite{xu2018representation} to derive 
\begin{equation*}
   \mathbb{E}\Big[\frac{\partial \mathbf{h}_{v}^{(m)}}{\partial\mathbf{h}_{u}^{(k)}}\Big] = \rho\prod_{s = k+1}^{m}\W^{(s)}(\oper^{m-k})_{vu}.
\end{equation*}

\noindent We can now proceed to prove the relation between sensitivity analysis and access time.

\begin{proof}[Proof of \Cref{thm:access}]
Under \Cref{assumption}, we can write the term $\mathbf{J}_k^{(m)}(v,u)$ as 
\begin{align*}
    \mathbb{E}\Big[\mathbf{J}_k^{(m)}(v,u)\Big] &= \mathbb{E}\Big[ \frac{1}{d_v}\frac{\partial \mathbf{h}_{v}^{(m)}}{\partial \mathbf{h}_{v}^{(k)}} - \frac{1}{\sqrt{d_v d_u}}\frac{\partial \mathbf{h}_{v}^{(m)}}{\partial \mathbf{h}_{u}^{(k)}}\Big] \\ 
    &= \rho\prod_{s = k+1}^{m}\W^{(s)}\Big(\frac{1}{d_v}(\oper^{m-k})_{vv} - \frac{1}{\sqrt{d_v d_u}}(\oper^{m-k})_{vu} \Big).
\end{align*}
\noindent 
\noindent Since $\oper = c_{\rs}\mathbf{I} + c_{\mpas}\mathbf{D}^{-1/2}\mathbf{A}\mathbf{D}^{-1/2}$, we can rely on the spectral decomposition of the graph Laplacian -- see the conventions and notations introduced in \Cref{app:sec_preliminaries} -- to write
\begin{equation*}
    \oper = \sum_{\ell = 0}^{n-1}\left(c_{\rs} + c_{\mpas}(1-\lambda_{\ell})\right)\eigen_{\ell}\eigen_{\ell}^\top,
\end{equation*}
\noindent where we recall that $\DELta \eigen_\ell = \lambda_\ell \eigen_\ell$. We can then bound (in {\bf expectation}) the Jacobian obstruction by
\begin{align*}
    \obst^{(m)}(v,u) &= \sum_{k=0}^{m}\| \mathbf{J}^{(m)}_k(v,u) \| \geq \sum_{k = 0}^{m} \rho\nu^{m-k}\Big|\Big(\frac{1}{d_v}(\oper^{m-k})_{vv} - \frac{1}{\sqrt{d_v d_u}}(\oper^{m-k})_{vu} \Big)\Big| \\
    &\geq \rho \Big|\sum_{k = 0}^{m} \nu^{m-k}\Big(\frac{1}{d_v}(\oper^{m-k})_{vv} - \frac{1}{\sqrt{d_v d_u}}(\oper^{m-k})_{vu} \Big)\Big| \\
    &= \rho \Big|\sum_{k = 0}^{m} \nu^{m-k}\sum_{\ell = 0}^{n-1}\Big(c_{\rs} + c_{\mpas}(1-\lambda_{\ell})\Big)^{m-k}\left(\frac{\eigen^{2}_\ell(v)}{d_v} - \frac{\eigen_\ell(v)\eigen_\ell(u)}{\sqrt{d_u d_v}}\right)\Big| \\
    &= \rho\Big| \sum_{\ell = 0}^{n-1}\Big(\sum_{k = 0}^{m} \nu^{m-k}\left(c_{\rs} + c_{\mpas}(1-\lambda_{\ell})\right)^{m-k}\Big)\Big(\frac{\eigen^{2}_\ell(v)}{d_v} - \frac{\eigen_\ell(v)\eigen_\ell(u)}{\sqrt{d_u d_v}}\Big)\Big|\\
    &= \rho\Big|\sum_{\ell = 1}^{n-1}\sum_{k = 0}^{m} \left(\nu(c_{\rs} + c_{\mpas}(1-\lambda_{\ell}))\right)^{m-k}\Big(\frac{\eigen^{2}_\ell(v)}{d_v} - \frac{\eigen_\ell(v)\eigen_\ell(u)}{\sqrt{d_u d_v}}\Big)\Big|,
\end{align*}
\noindent where in the last equality we have used that $\eigen_0(v) = \sqrt{d_v}/(2\lvert \E\rvert)$ for each $v\in\V$. We can then expand the geometric sum by using our assumption $\nu(c_{\rs} + c_{\mpas}) = 1$ and write
\begin{equation*}
\obst^{(m)}(v,u) \geq \rho\Big|\sum_{\ell = 1}^{n-1}\frac{1 - \left(\nu(c_{\rs} + c_{\mpas}(1-\lambda_\ell))\right)^{m+1}}{1 - \nu(c_{\rs} + c_{\mpas}) +\nu c_{\mpas}\lambda_\ell}\Big(\frac{\eigen^{2}_\ell(v)}{d_v} - \frac{\eigen_\ell(v)\eigen_\ell(u)}{\sqrt{d_u d_v}}\Big)\Big|. 
\end{equation*}
\noindent Since $\nu(c_{\rs} + c_{\mpas}) = 1$, we can simplify the lower bound as
\begin{equation*}
\obst^{(m)}(v,u) \geq \rho\Big|\sum_{\ell = 1}^{n-1}\frac{1}{\nu c_{\mpas}\lambda_\ell} \Big(\frac{\eigen^{2}_\ell(v)}{d_v} - \frac{\eigen_\ell(v)\eigen_\ell(u)}{\sqrt{d_u d_v}}\Big)\Big| - \rho\Big| \sum_{\ell = 1}^{n-1}\frac{\left(\nu(c_{\rs} + c_{\mpas}(1-\lambda_\ell))\right)^{m+1}}{\nu c_{\mpas}\lambda_\ell}\Big(\frac{\eigen^{2}_\ell(v)}{d_v} - \frac{\eigen_\ell(v)\eigen_\ell(u)}{\sqrt{d_u d_v}}\Big) \Big|. 
\end{equation*}
\noindent By \citet[Theorem 3.1]{lovasz1993random}, the first term is equal to $\lvert (\nu c_{\mpas})^{-1}\mathsf{t}(u,v)/2\lvert\E\rvert \rvert$ which is a positive number. Concerning the second term, we recall that the eigenvalues of the graph Laplacian are ordered from smallest to largest and that $\eigen_\ell$ is a unit vector, so 
\begin{equation*}
\obst^{(m)}(v,u) \geq \frac{\rho}{\nu c_{\mpas}}\frac{\mathsf{t}(u,v)}{2\lvert \E\rvert} - \frac{\rho(1 - \nu c_{\mpas}\lambda^\ast)^{m+1}}{\nu c_{\mpas}\lambda_1}\frac{n-1}{d_{\mathrm{min}}},
\end{equation*}
\noindent with $\lambda^\ast$ such that $\lvert 1 - \lambda^\ast \rvert = \max_{\ell > 0}\lvert 1 - \lambda_\ell \rvert$ which completes the proof.

\end{proof}
\begin{proof}[Proof of \Cref{thm:effective_resistance}]
We follow the same strategy used in the proof of \Cref{thm:access}. Under \Cref{assumption}, we can write the term $\mathbf{J}_k^{(m)}(v,u)$ as 
\begin{align*}
    \mathbb{E}\Big[\mathbf{J}_k^{(m)}(v,u)\Big] &= \mathbb{E}\Big[ \frac{1}{d_v}\frac{\partial \mathbf{h}_{v}^{(m)}}{\partial \mathbf{h}_{v}^{(k)}} - \frac{1}{\sqrt{d_v d_u}}\frac{\partial \mathbf{h}_{v}^{(m)}}{\partial \mathbf{h}_{u}^{(k)}} + \frac{1}{d_u}\frac{\partial \mathbf{h}_{u}^{(m)}}{\partial \mathbf{h}_{u}^{(k)}}  - \frac{1}{\sqrt{d_v d_u}}\frac{\partial \mathbf{h}_{u}^{(m)}}{\partial \mathbf{h}_{v}^{(k)}}\Big] \\
    &= \rho\prod_{s = k+1}^{m}\W^{(s)}\Big(\frac{1}{d_v}(\oper^{m-k})_{vv} + \frac{1}{d_u}(\oper^{m-k})_{uu} - 2(\oper^{m-k})_{vu} \Big)
\end{align*}
\noindent where we have used the symmetry of $\oper$. We note that the term within brackets can be equivalently reformulated as
\begin{equation*}
    \frac{1}{d_v}(\oper^{m-k})_{vv} + \frac{1}{d_u}(\oper^{m-k})_{uu} - 2(\oper^{m-k})_{vu} = \langle \frac{\mathbf{e}_v}{\sqrt{d_v}} - \frac{\mathbf{e}_u}{\sqrt{d_u}}, \oper^{m-k}\Big(\frac{\mathbf{e}_v}{\sqrt{d_v}} - \frac{\mathbf{e}_u}{\sqrt{d_u}}\Big)\rangle
\end{equation*}
\noindent where $\mathbf{e}_v$ is the vector with $1$ at entry $v$, and zero otherwise. In particular,
we note an {\bf important fact}: since, by assumption, $c_{\rs} \geq c_{\mpas}$ and $\lambda_{n-1} < 2$, whenever $\gph$ is not bipartite, we derive that $\oper$ is a {\em positive definite operator}. We can then bound (in {\bf expectation}) the Jacobian obstruction by
\begin{align*}
    \tilde{\obst}^{(m)}(v,u) &= \sum_{k=0}^{m}\| \mathbf{J}^{(m)}_k(v,u) \| \leq \sum_{k = 0}^{m} \rho\mu^{m-k}\sum_{\ell = 0}^{n-1}\left(c_{\rs} + c_{\mpas}(1-\lambda_{\ell})\right)^{m-k}\Big(\frac{\eigen_\ell(v)}{\sqrt{d_v}} - \frac{\eigen_\ell(u)}{\sqrt{d_u}}\Big)^{2} \\
    &= \rho\sum_{\ell = 0}^{n-1}\left(\sum_{k = 0}^{m} \mu^{m-k}\left(c_{\rs} + c_{\mpas}(1-\lambda_{\ell})\right)^{m-k}\right)\Big(\frac{\eigen_\ell(v)}{\sqrt{d_v}} - \frac{\eigen_\ell(u)}{\sqrt{d_u}}\Big)^{2} \\
    &= \rho\sum_{\ell = 1}^{n-1}\left(\sum_{k = 0}^{m} \mu^{m-k}\left(c_{\rs} + c_{\mpas}(1-\lambda_{\ell})\right)^{m-k}\right)\Big(\frac{\eigen_\ell(v)}{\sqrt{d_v}} - \frac{\eigen_\ell(u)}{\sqrt{d_u}}\Big)^{2},
\end{align*}
\noindent where in the last equality we have used that $\eigen_0(v) = \sqrt{d_v}/(2\lvert \E\rvert)$. We can then expand the geometric sum by using our assumption $\mu(c_{\rs} + c_{\mpas})\leq 1$ and write
\begin{align*}
\tilde{\obst}^{(m)}(v,u) &\leq \rho\sum_{\ell = 1}^{n-1}\frac{1 - \left(\mu(c_{\rs} + c_{\mpas}(1-\lambda_\ell))\right)^{m+1}}{1 - \mu(c_{\rs} + c_{\mpas}) +\mu c_{\mpas}\lambda_\ell}\Big(\frac{\eigen_\ell(v)}{\sqrt{d_v}} - \frac{\eigen_\ell(u)}{\sqrt{d_u}}\Big)^{2} \\
&\leq \sum_{\ell = 1}^{n-1}\frac{\rho}{\mu c_{\mpas}\lambda_\ell}\Big(\frac{\eigen_\ell(v)}{\sqrt{d_v}} - \frac{\eigen_\ell(u)}{\sqrt{d_u}}\Big)^{2} \\
& = \frac{\rho}{\mu c_{\mpas}}\sum_{\ell = 1}^{n-1}\frac{1}{\lambda_\ell}\Big(\frac{\eigen_\ell(v)}{\sqrt{d_v}} - \frac{\eigen_\ell(u)}{\sqrt{d_u}}\Big)^{2} \\
& = \frac{\rho}{\mu c_{\mpas}}\res(v,u)
\end{align*} 
\noindent where in the last step we used the spectral characterization of the effective resistance derived in \citet{lovasz1993random} -- which was also leveraged in \citet{arnaiz2022diffwire} to derive a novel rewiring algorithm. Since by \citet{chandra1996electrical} we have $2\res(v,u)\lvert\E\rvert = \tau(v,u)$, this completes the proof of the upper bound. The lower bound case follows by a similar argument. In fact, one arrives at the estimate
\begin{align*}
\tilde{\obst}^{(m)}(v,u) &\geq \rho\sum_{\ell = 1}^{n-1}\frac{1 - \left(\nu(c_{\rs} + c_{\mpas}(1-\lambda_\ell))\right)^{m+1}}{1 - \nu(c_{\rs} + c_{\mpas}) +\nu c_{\mpas}\lambda_\ell}\Big(\frac{\eigen_\ell(v)}{\sqrt{d_v}} - \frac{\eigen_\ell(u)}{\sqrt{d_u}}\Big)^{2}. 
\end{align*} 
\noindent  We derive
 \[
 1 - \left(\nu(c_{\rs} + c_{\mpas}(1-\lambda_\ell))\right)^{m+1} \geq 1 - \left(\nu(c_{\rs} + c_{\mpas}(1-\lambda^\ast))\right)^{m+1},
 \]
 \noindent where $\lvert 1 -\lambda^\ast\rvert = \max_{\ell > 0} \lvert 1 - \lambda^\ast\rvert$.
 \noindent Next, we also find that 
 \[
 \frac{1}{1 - \nu(c_{\rs} + c_{\mpas}) +\nu c_{\mpas}\lambda_\ell} \geq \frac{\epsilon}{\nu c_{\mpas}\lambda_\ell} \Longleftrightarrow \lambda_\ell \geq \frac{\epsilon}{1-\epsilon}\, \frac{1-\nu (c_{\rs} + c_{\mpas})}{\nu c_{\mpas}}. 
 \]
 \noindent Since the eigenvalues are ordered from smallest to largest, it suffices that 
  \[
 \lambda_1 \geq \frac{\epsilon}{1-\epsilon}\, \frac{1-\nu (c_{\rs} + c_{\mpas})}{\nu c_{\mpas}} \Longleftrightarrow \epsilon \leq \epsilon_\gph := \frac{\lambda_1}{\lambda_1 + \frac{1 - \nu(c_{\rs}+c_{\mpas})}{\nu c_{\mpas}}}.
 \]
 \noindent This completes the proof.
\end{proof}

\noindent {\em We emphasize that without the degree normalization, the bound would have an extra-term (potentially diverging with the number of layers) and simply proportional to the degrees of nodes $v,u$. The extra-degree normalization is off-setting this uninteresting contribution given by the steady state of the Random Walks.}

\section{Graph Transfer}
\label{app:graph-transfer}

The goal in the three graph transfer tasks - $\mathsf{Ring}$, $\mathsf{CrossedRing}$, and $\mathsf{CliquePath}$ - is for the $\MPNN$ to `transfer' the features contained at the target node to the source node. $\mathsf{Ring}$ graphs are cycles of size $n$, in which the target and source nodes are placed at a distance of $\lfloor n / 2 \rfloor$ from each other. $\mathsf{CrossedRing}$ graphs are also cycles of size $n$, but now include `crosses' between the auxiliary nodes. Importantly, the added edges do not reduce the minimum distance between the source and target nodes, which remains $\lfloor n / 2 \rfloor$. $\mathsf{CliquePath}$ graphs contain a  $\lfloor n / 2  \rfloor$-clique and a path of length $\lfloor n / 2  \rfloor$. The source node is placed on the clique and the target node is placed at the end of the path. The clique and path are connected in such a way that the distance between the source and target nodes is $\lfloor n / 2  \rfloor + 1$, in other words the source node requires one hop to gain access the path.

\begin{figure}[!htbp]
    \centering
    \includegraphics[width=\textwidth]{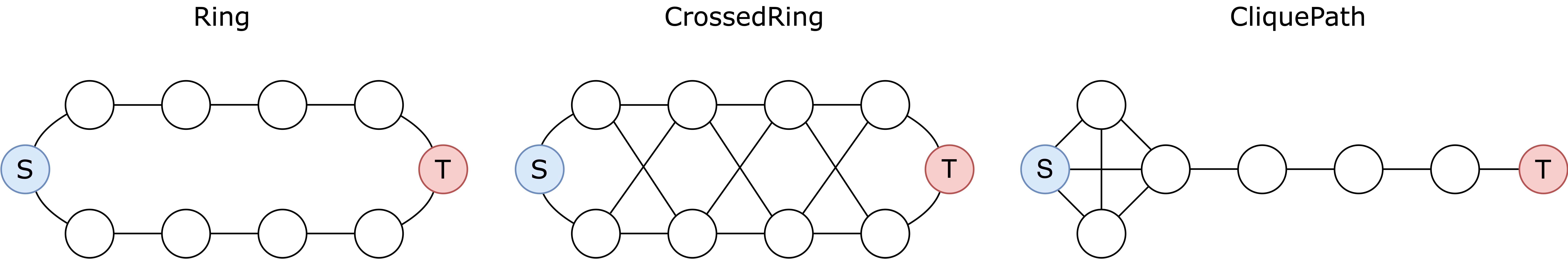}
    \caption{Topological structure of $\mathsf{RingTransfer}$, $\mathsf{CrossedRingTransfer}$, and $\mathsf{CliquePath}$. The nodes marked with an $S$ are the source nodes, while the nodes with a $T$ are the target nodes. All tasks are shown for a distance between the source and target nodes of $r=5$.}
    \label{fig:graph-transfer-example}
\end{figure}

Figure \ref{fig:graph-transfer-example} shows examples of the graphs contained in the $\mathsf{Ring}$, $\mathsf{CrossedRing}$, and $\mathsf{CliquePath}$ tasks, for when the distance between the source and target nodes is $r=5$. In our experiments we take as input dimension $p=5$ and assign to the target node a randomly one-hot encoded feature vector - for this reason the random guessing baseline obtains $20\%$ accuracy. The source node is assigned a vector of all $0s$ and the auxiliary nodes are instead assigned vectors of $1s$. Following \cite{bodnar2021weisfeilercell}, we generate $5000$ graphs for the training set and $500$ graphs for the test set for each task. In our experiments, we report the mean accuracy over the test set. We train for $100$ epochs, with depth of the $\MPNN$ equal to the distance between the source and target nodes $r$. Unless specified otherwise, we set the hidden dimension to $64$. During training and testing, we apply a mask over all the nodes in order to focus only on the source node to compute losses and accuracy scores.


\section{Signal Propagation}
\label{app:signal-prop}

\begin{figure}[t]
    \centering
    \includegraphics[width=.75\textwidth]{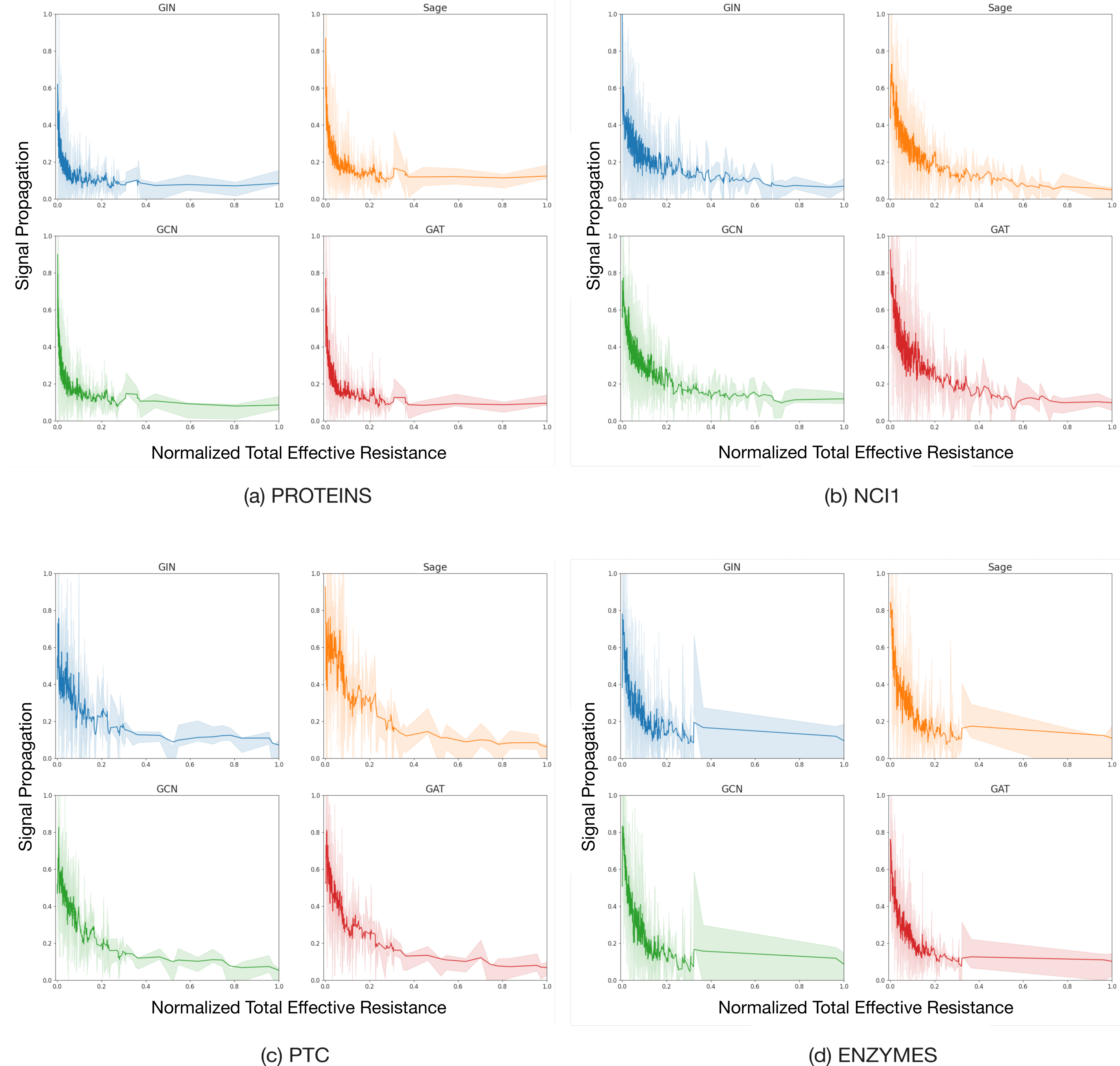}
    \caption{Decay of the amount of information propagated through the graphs w.r.t. the normalized total effective resistance (commute time) for: (a) $\mathsf{PROTEINS}$; (b) $\mathsf{NCI1}$; (c) $\mathsf{PTC}$; (d) $\mathsf{ENZYMES}$. For each dataset we report the decay for: (i) $\mathsf{GIN}$ (top-left); (ii) $\mathsf{Sage}$ (top-right), (iii) $\mathsf{GCN}$  (bottom-left) and (iv)
    $\mathsf{GAT}$(bottom-right).}
    \label{fig:signal-diffusion-main}
\end{figure}

In this section we provide synthetic experiments on the $\mathsf{PROTEINS}$, $\mathsf{NCI1}$, $\mathsf{PTC}$, $\mathsf{ENZYMES}$ datasets with the aim to provide empirical evidence to the fact that the total effective resistance of a graph, $\res_{\gph} = \sum_{v,u} \res{(v,u)}$ \cite{ellens2011effective}, is related to the ease of information propagation in an $\MPNN$. The experiment is designed as follows: we first fix a source node $v\in\V$ assigning it a $p$-dimensional unitary feature vector, and assigning the rest of the nodes zero-vectors. We then 
consider the quantity 
\begin{equation*}
h^{(m)}_\odot = \frac{1}{p \max_{u\neq v} d_\gph (v,u)}  \sum_{f=1}^{p}\sum_{u \neq v}\frac{h_u^{(m), f}}{\| h_u^{(m),f} \|} d_\gph(v,u),
\end{equation*}

to be the amount of signal (or `information') that has been propagated through $\gph$ by an $\MPNN$ with $m$ layers. Intuitively, we measure the (normalized) propagation distance over $\gph$, and average it over all the $p$ output channels. By propagation distance we mean the average distance to which the initial `unit mass' has been propagated to - a larger propagation distance means that on average the unit mass has travelled further w.r.t. to the source node. The goal is to show that $h^{(m)}_\odot$ is {\em inversely proportional to $\res_{\gph}$}. In other words, we expect graphs with \emph{lower} total effective resistance to have a \emph{larger} propagation distance. The experiment is repeated for each graph $\gph$ that belongs to the dataset $\mathcal{D}$. We start by randomly choosing the source node $v$, we then set $\mathbf{h}_v$ to be an arbitrary feature vector with unitary mass (i.e. $\| \mathbf{h}_v \|_{L_1} = 1$) and assigning the zero-vector to all other nodes (i.e. $\mathbf{h}_u = \boldsymbol{0}, \; u \neq v$). We use $\MPNN$s with
a number of layers $m$ close to the average diameter of the graphs in the dataset, input and hidden dimensions $p=5$ and ReLU activations. In particular, we estimate the resistance of $\gph$  by sampling $10$ nodes with uniform probability for each graph and we report $h^{(m)}_\odot$ accordingly.  In Figure \ref{fig:signal-diffusion-main} we show that {\em $\MPNN$s are able to propagate information further when the effective resistance is low},  validating empirically the impact of the graph topology on over-squashing phenomena. It is worth to emphasize that in this experiment, the parameters of the $\MPNN$ are randomly initialized and there is no underlying training task. This implies that in this setup we are isolating the problem of propagating the signal throughout the graph, separating it from vanishing gradient phenomenon.

\end{document}